\documentclass{article}

\usepackage{microtype}
\usepackage{graphicx}
\usepackage{subcaption}
\usepackage{booktabs} 

\usepackage{hyperref}

\usepackage[accepted]{icml2022}

\usepackage{natbib}
\usepackage{amsmath}
\usepackage{amssymb}
\usepackage{mathtools}
\usepackage{amsthm}
\usepackage{graphicx}
\usepackage{tikz}
\usetikzlibrary{arrows,automata,angles}

\DeclarePairedDelimiter{\abs}{\lvert}{\rvert}
\DeclarePairedDelimiter{\norm}{\lVert}{\rVert}

\def\Snospace~{\S{}}

\newcommand{\new}{\textcolor{red}}

\theoremstyle{definition}
\newtheorem{example}{Example}

\usepackage{paralist}

\newtheorem{proposition}{Proposition}
\newtheorem{lemma}{Lemma}

\icmltitlerunning{Extracting Finite Automata from RNNs Using State Merging}
\begin{document}

\twocolumn[
\icmltitle{Extracting Finite Automata from RNNs Using State Merging}



\icmlsetsymbol{equal}{*}

\begin{icmlauthorlist}
\icmlauthor{William Merrill}{equal,cds}
\icmlauthor{Nikolaos Tsilivis}{equal,cds}
\end{icmlauthorlist}

\icmlaffiliation{cds}{Center for Data Science, NYU}

\icmlcorrespondingauthor{William Merrill}{willm@nyu.edu}
\icmlcorrespondingauthor{Nikolaos Tsilivis}{nt2231@nyu.edu}

\icmlkeywords{Machine Learning, ICML, Finite Automata, NLP, Interpretability}

\vskip 0.3in
]



\printAffiliationsAndNotice{\icmlEqualContribution} 

\begin{abstract}
    One way to interpret the behavior of a blackbox recurrent neural network (RNN) is to extract from it a more interpretable discrete computational model, like a finite state machine, that captures its behavior. In this work, we propose a new method for extracting finite automata from RNNs inspired by the state merging paradigm from grammatical inference. We demonstrate the effectiveness of our method on the Tomita languages benchmark, where we find that it is able to extract faithful automata from RNNs trained on all languages in the benchmark. We find that extraction performance is aided by the number of data provided during the extraction process, as well as, curiously, whether the RNN model is trained for additional epochs after perfectly learning its target language. We use our method to analyze this phenomenon, finding that training beyond convergence is useful because it leads to compression of the internal state space of the RNN. This finding demonstrates how our method can be used for interpretability and analysis of trained RNN models.
\end{abstract}

\section{Introduction}

Interpretability poses a problem for deep-learning based sequence models like recurrent neural networks (RNNs). When trained on language data and other structured discrete sequences, such models implicitly acquire structural rules over the input sequences that modulate their classification decisions. However, it is often difficult to recover the discrete rules encoded in the parameters of the network. Traditionally, it is useful to be able to understand the rules a model is using to reach its classification decision, as models like decision trees or finite automata allow. Not only does this address practical deployment concerns like being able to explain model decisions or debug faulty inputs, but it also has more foundational scientific uses: e.g., inducing a model of grammar over natural sentences can help test and build linguistic theories of the syntax of natural language. Another potential use of extracted automata is to study the training dynamics of neural networks: automata can be extracted from different checkpoints and compared to understand how the strategy of an RNN evolves over training. In \autoref{sec:beyond_conv}, we will discuss analysis of our method that may provide insight on the implicit regularization of RNN training.

How, then, can we gain insight into the discrete rules acquired by RNN models? One family of approaches is to extract finite automata that capture the behavior of an RNN, and use the extracted state machine for interpretability. This problem can be seen as a special case of grammar induction \citep{Gold_1978, Angluin1987LearningRS}: the task of, given samples from a formal language, inferring a small automaton that will generate the data. Thus, past work on RNN extraction has generally adapted techniques from the grammar induction literature. For example, \citet{weiss2018extracting} do this by adapting the $L^*$ query learning algorithm for grammar induction to work with an RNN oracle. \citet{lacroce2021extracting} leverage spectral learning, a different framework for grammar induction, to infer weighted automata from RNNs, as opposed to the more standard deterministic automata. Other work has used $k$-means clustering on the RNN hidden states to extract a graph of states \citep{wang2018empirical}.

$L^*$ is an active learning approach that learns via queries of two forms: membership of strings in $L$, and equivalence queries comparing a candidate DFA and the true underlying DFA. Thus, \citet{weiss2018extracting} assume blackboxes computing these oracles are available at train time, which may be problematic for the potentially expensive equivalence queries. The $k$-means method of \citet{wang2018empirical} does not have this problem, although it comes with no theoretical guarantees of faithfulness, and requires that the number of states must be picked as a hyperparameter. In this work, we will present an alternative extraction method that does not require expensive equivalence queries, and where the number of states does not need to be set in advance.

To meet these goals, we will leverage \emph{state merging} \citep{oncina1992inferring, lang1998results, sebban2003state}, another grammar induction paradigm, to extract deterministic finite automata (DFAs) from RNNs. State merging works by first building a prefix tree from a finite dataset: a deterministic automaton that simply memorizes the training data, and will not recognize any held-out strings beyond the finite set used to build the prefix tree. The next step of the process is to compress this prefix tree by merging states together, using a strategy (or `policy') $\pi$. This process both reduces the automaton size and introduces loops between states. Through this, the automaton gains the ability to generalize to an infinite set of held-out strings. Of course, the nature of this generalization depends on how $\pi$ is computed. For grammatical inference, $\pi$ is generally computed by verifying simple constraints are met: in order to merge two states, the states must agree in whether or not they are final states. We will add an additional constraint that the RNN representations associated with each state must be close, thus enforcing that our learned automaton reflects the structure of the RNN's implicit state space.



In summary, we introduce state merging as a method to extract DFAs from blackbox RNNs.
We first show in \autoref{subsec:extr_faith} that our state merging method enables RNN extraction on all 7 Tomita languages \citep{tomita1982dynamic}, the standard benchmark for evaluating RNN extraction. As an additional contribution, we use our method to show that continuing to train an RNN past convergence in development accuracy makes it easier to extract a DFA from it, and the implicit state space of the resulting DFA is simplified (\autoref{sec:beyond_conv}). We discuss speculatively how this phenomenon may have implications for understanding the implicit regularization of RNN training.

\section{Background} \label{sec:background}

\subsection{Recurrent Neural Networks} \label{sec:rnns}

For our purposes, a generalized RNN is a function mapping a sequence of symbols $\{w_i\}_{i=1}^n$ to a sequence of labels $\{y_i\}_{i=1}^n$. In the abstract, the RNN has a state vector $\mathbf h_i \in \mathbb{R}^d$ that satisfies the following form for some gating function $f$:
\begin{align*}
    \mathbf 
    \mathbf h_{i + 1} &= f(\mathbf h_i, w_{i + 1}) \\
    \mathbf y_{i + 1} &= \mathrm{argmax}(\mathbf w^\top \mathbf h_{i + 1} + b) .
\end{align*}
\noindent In principle, our method can be applied to RNNs with any gating function $f$, but, in the paper, we will use the simple recurrent gating \citep{elman1990finding}:
\begin{equation}\label{eq:state_transition}
    \mathbf h_{i + 1} = \tanh(U \mathbf h_i + V \mathbf x_{i+1}) ,
\end{equation}
\noindent where $\mathbf x_i$ is a vector embedding of token $w_i$. Other common variants include Long Short-Term Memory networks \citep[LSTMs;][]{sepp1997lstm} and Gated Recurrent Units \citep[GRUs;][]{cho2014properties}.

\subsection{Deterministic Finite Automata} \label{sec:dfas}
Automata have a long history of study in theoretical computer science, linguistics, and related fields, originally having been formalized in part as a discrete model of neural networks \citep{kleene1956automata, minsky1956some}.
A deterministic finite automaton (DFA) can be specified as a tuple $A = \langle \Sigma, Q, q_0, \delta, F\rangle$, where:
\begin{compactitem}
    \item $\Sigma$ is a finite input alphabet (set of tokens);
    \item $Q$ is a set of states, along with a special ``undefined'' state $\emptyset$, where $\emptyset \not\in Q$;
    \item $q_0 \in Q$ is an initial state;
    \item $\delta : (Q \cup \{ \emptyset \}) \times \Sigma \to (Q \cup \{ \emptyset \})$ is a transition function such that $\forall \sigma \in \Sigma$, $\delta(\emptyset, \sigma) = \emptyset$;
    \item $F \subseteq Q$ is a set of accepting states.
\end{compactitem}

Now that we have formally specified this model, how does one do computation with it? Informally, when processing a string $w \in \Sigma^n$, $A$ starts in state $q_0$, and each token in the input string causes it to transition to a different state according to $\delta$. Once all input tokens have been consumed, the machine either accepts or rejects the input string depending on whether the final state $q_n \in F$. More formally, we define the state after the prefix $w_{:i}$ as:
\begin{align*}
    q_i = \delta(q_{i-1}, w_i) .
\end{align*}
\noindent We then say that $A$ accepts a string $w \in \Sigma^n$ if and only if $q_n \in F$. The regular language recognized by $A$ is the set of strings it accepts, i.e.,
\begin{equation*}
    L(A) = \{w \mid q_{n}(w) \in F\} .
\end{equation*}

\begin{example}
    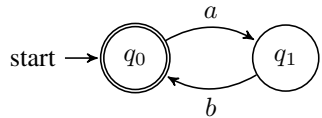
\begin{figure}
    \centering
    \begin{tikzpicture}[->,>=stealth',shorten >=1pt,auto,node distance=2cm, semithick]
    \node[initial,state,accepting] (2) {$q_0$};
    \node[state,right of=2] (3) {$q_1$};
    
    \path (2) edge [bend left, above] node {$a$} (3);
    \path (3) edge [bend left, below] node {$b$} (2);
    \end{tikzpicture}
    \caption{A DFA specified as $\langle \Sigma, Q, q_0, \delta, F\rangle$ with $\Sigma = \{a, b\}$,  $Q = \{q_0, q_1\}$, $\delta \text{ such that } \delta(q_0, a) = q_1, \delta(q_1, b) = q_0$, and $F = \{q_0\} $. It recognizes the language $(ab)^* = \{\epsilon, ab, abab, ababab, \cdots\}$. The $*$ symbol denotes Kleene star, i.e., $0$ or more repetitions of a string.}
    \label{fig:dfa_example}
    \end{figure}
    Consider the DFA in \autoref{fig:dfa_example}. It recognizes the language $(ab)^\star = \{\epsilon, ab, abab, ababab, \cdots\}$, and it is the minimal automaton that does so.
    \begin{itemize}
        \item For the string $ab$, the computation would start from $q_0$ (the initial state - common for any string), then the automaton would traverse to $q_1$ ($\delta(q_0, a) = q_1$), and, subsequently, to $q_0$ ($\delta(q_1, b) = q_0$). Since the final state $q_1$ belongs to $F$ after consuming all input tokens, we say that the DFA accepts $ab$.
        \item For the string $aba$, the DFA identically reaches state $q_0$ after consuming the prefix $ab$. However, the final $a$ causes the DFA to transition back to $q_1$. Because $q_1 \not\in F$, the DFA rejects $aba$.
        \item For the string $abb$, the DFA also reaches state $q_0$ after consuming the prefix $ab$. At this point, the transition $\delta(q_0, b) = \emptyset$, so the state will be $\emptyset$ (``error'') for the rest of the string. Since $\emptyset \not\in F$, the DFA rejects $abb$.
    \end{itemize}
\end{example}   

\subsection{Power of DFAs}

DFAs are equivalent to nondeterministic finite automata, both recognizing the regular languages \citep{kleene1956automata}. The regular languages form the lowest level of the Chomsky hierarchy \citep{Chomsky1956Three}, and intuitively represent languages that can be recognized with memory that does not grow with the sequence length. In contrast, more powerful classes allow the memory of the recognizer to grow with the length of the input string. For example, context-free languages correspond to nondetermistic finite automata augmented with a stack data structure \citep{Chomsky1956Three}, enabling $O(n)$ memory on strings of length $n$.
Other classes in the Chomsky hierarchy correspond to the languages recognizable by even more complex automata: for example, the recursive languages correspond to the set of languages that can be recognized by a Turing machine.

\subsection{Connections of RNNs to DFAs}

At a high level, DFAs and RNNs can both be used to match the language recognition task specification: essentially, binary classification over strings. RNNs with continuous activation functions and unbounded computation time and precision have been shown to be Turing-complete, meaning they can recognize languages that are not regular \citep{Siegelmann1992OnTC}. However, more recent literature has argued that these assumptions differ substantially from the type of RNNs trained in modern deep learning \citep{weiss2018, merrill-2019-sequential}. The same work suggests that the regular languages are a much more reasonable model for the capacity of RNNs as trainable deep learning model. We now briefly summarize this line of research.

Some of the original motivation for formalizing finite automata came from trying to develop a model of computation for early connectionist versions of neural networks \citep{kleene1956automata, minsky1956some}. Thus, by design, RNNs with threshold activation functions are equivalent in terms of the set of languages they can recognize to finite automata \citep{merrill-etal-2020-formal}. More recent work has shown that the infinite parameter norm limits of simple RNNs and GRUs are equivalent to finite automata in expressive power \citep{merrill-2019-sequential}, and found that language learning experiments with these networks can often be predicted by the theoretical capacity of these ``saturated'' infinite-norm networks \citep{merrill-etal-2020-formal}. For instance, \citet{weiss2018} found that RNNs and GRUs cannot ``count'' (a capability requiring more than finite state), unlike the more complicated LSTM. Combining this theoretical and empirical evidence suggests that simple RNNs and GRUs behave as finite-state models, rather than models whose states grow with the input length. This perspective supports using deterministic finite-state automata as a target for extraction with RNNs.

We note, however, that if we would like to do extraction for LSTMs or other complex RNNs, it could make sense to extract a counter automaton \citep{fischer1968counter} rather than a finite automaton, which we believe state merging could be adapted to accommodate in future work.

\section{Method} \label{sec:method}

We now describe our state merging method for extracting DFAs from RNNs. Our method assumes a blackbox RNN model that supplies the following desiderata:
\begin{enumerate}
    \item \textbf{Hidden States} Given an input string $x \in \Sigma^n$, and for each $0 \leq i \leq n$, the RNN produces a vector $\mathbf h_i \in \mathbb{R}^k$ that encodes the full state of the model after processing the prefix of $w$ up to index $i$. Thus, $\mathbf h_0$ corresponds to a representation for the empty string $\epsilon$. We will write $\mathbf H$ to mean the full $(n + 1) \times k$ hidden state matrix.
    \item \textbf{Recognition Decisions} Given an input string $x \in \Sigma^n$, the RNN produces a vector $\mathbf{\hat y} \in (0,1)^{n+1}$ that scores the probability that \emph{each prefix} of $x$ is a valid string in the formal language defined by the RNN.
\end{enumerate}
Our method can be applied to any model satisfying these properties. However, as discussed in the previous section, it is most motivated to apply it to simple RNNs or GRUs, which have been shown to resemble finite state machines. If our method is applied to an LSTM or other complex RNN variant, the extracted DFA will potentially be a finite-state approximation of more complex behavior.

Our state merging algorithm has two parts: first, we construct a prefix tree using the recognition decisions $\mathbf{\hat y}$. Next, we merge states in the prefix tree according to the RNN hidden states $\mathbf{H}$.

\subsection{Building the Prefix Tree}

A prefix tree, or trie, is a DFA that can be built to correctly recognize any language $L$ over all prefixes of a finite support of strings $\{w_i\}^m_{i=1}$. Each state in the tree represents a prefix of some $w_i$, and is labelled according to whether that prefix is a valid string in $L$. Paths of transitions are added to the tree to connect prefixes together in the natural way, e.g., $w_i=ab$ would induce three states $q_\epsilon, q_a, q_{ab}$ and the path $q_\epsilon \to_a q_a \to_b q_{ab}$ (see top row in \autoref{fig:dfa_img} for an example).

To build the prefix tree, we sample a new training set of strings $\{w_i\}^m_{i=1}$, and record as labels $\mathbf{\hat y}(w_i)$, i.e., whether each prefix of every $w_i$ is a valid string in $L$. Note that this training set is distinct (and generally much smaller) than the training set used to train the RNN.
After its construction, we identify with each state $q_j$ a feature vector $\phi(q_j) = \mathbf h_{\lvert w\rvert}(w)$, where $w$ is the prefix corresponding to $q_j$.


\subsection{Merging States}

Once the prefix tree is built, we define a policy $\pi(q_i, q_j)$ that compares states, and predicts whether or not to merge them. Let $\kappa \in (0, 1)$ be a hyperparameter. We specify $\pi$ to merge $q_i \to q_j$ when both of the following two constraints are met:
\begin{enumerate}
    \item \textbf{Consistency:} $q_i \in F \iff q_j \in F$
    \item \textbf{Similarity:} $\mathrm{cos}(\phi(q_i), \phi(q_j)) > 1 - \kappa$ 
\end{enumerate}
The consistency constraint is standard in grammar induction: it guarantees that each step preserves the performance of the automaton across observed positive and negative examples, and thus that the new automaton is consistent with the behavior of the RNN on the training set. We add the similarity constraint to enforce that the automaton's representations reflect the true structure of the underlying state space in the RNN. Thus, two states are merged if and only if doing so would preserve recognition behavior on the training set and reflects the internal structure of the RNN state space.

If both of these conditions are met, then we merge $q_i \to q_j$. To do this, we delete the state $q_i$ from the graph, and choose the transitions to/from $q_j$ by taking the union of all transitions involving $q_i$ or $q_j$.\footnote{This procedure may yield a \emph{non-deterministic finite automaton} which is equivalent though to a DFA \citep{kleene1956automata}.}.
This merge operation is not fully symmetric, since the representation $\phi(q')$ is inherited from $q_j$ after merging $q_i$ and $q_j$. On the other hand, the conditions to merge two states are defined symmetrically. Thus, the algorithm will potentially reach different results depending on the enumeration order for $q_i$ and $q_j$. In practice, this will not be an issue, as long as $\kappa$ is set sufficiently high, since in this case, the vectors $\phi(q_i)$ and $\phi(q_j)$ will be effectively equivalent from the point of view of the algorithm.


\subsection{Postprocessing}

Finally, after reducing the automaton via state merging, we can apply a DFA minimization step \citep{hopcroft1971n} to reduce the size of the extracted DFA while preserving the language it recognizes. DFA minimization is an operation that takes a regular language defined by a DFA and returns the DFA with the smallest number of states that recognizes that regular language, which is unique up to isomorphism. Thus, DFA minimization is semantically different than state merging: while state merging preserves recognition decisions over the training set, minimization is guaranteed to preserve recognition decisions over all strings. Thus, applying minimization alone to the initial prefix tree is not able to produce a DFA that generalizes beyond the training set. Our goal in applying minimization after state merging is to make the behavior of the resulting automaton easier to visualize and evaluate without changing it.

\subsection{Theoretical Motivation}
Our proposed algorithm is justified in the sense that it extracts the state transitions of the saturated version of the RNN it receives as input.
A saturated RNN with a $\tanh$ non-linearity will have state vectors $\mathbf{h} \in \{\pm 1 \}^d$. Thus, a saturated RNN has a finite number of states over which the update rule \eqref{eq:state_transition} acts as DFA transition function \citep[cf.][]{merrill-2019-sequential}.
Given the discontinuity of the RNN state space, a cosine similarity greater than $\frac{d-1}{d}$ ensures that two state vectors are the same. Trained RNNs have been found to become approximately saturated \citep{karpathy2015rnns}, suggesting the saturated network should closely capture their behavior. The following proposition, whose details and proof can be found in the Appendix, captures this intuition, while it also provides a way to select the similarity hyperparameter $\kappa$ based on the level of the saturation of the RNN.


\begin{proposition}
Let $\mathbf{h}_1, \mathbf{h}_2 \in \mathbb{R}^d$ be two normalized state vectors, $\Tilde{\mathbf{h}}_1, \Tilde{\mathbf{h}}_2 \in \{\pm 1\}^d$ their saturated versions and assume that the RNN is $\epsilon$-saturated with respect to these states, i.e., $\norm{\mathbf h_i - \Tilde{\mathbf{h}}_i}_2 \leq \epsilon$, $i \in \{1, 2\}$. Then, if $\cos(\mathbf h_1, \mathbf h_2) \geq 1 - \kappa$ with $\sqrt{\kappa} < \sqrt{2} \left( \frac{1}{\sqrt{d}} - \epsilon \right)$,
the two vectors represent the same state on the DFA / saturated RNN ($\Tilde{\mathbf{h}}_1 = \Tilde{\mathbf{h}}_2$).
\end{proposition}

In practice, one can measure the level of saturation \citep{merrill2020parameter} and select $\kappa$ from the expression above, but in our experiments we found that is not necessary, as a very small value of $\kappa$, together with the postprocessing step of DFA minimization suffices for successful DFA extraction.

\section{Data and Models} \label{sec:data}

\subsection{Tomita Languages}

The Tomita languages are a standard formal language benchmark used for evaluating grammar induction systems \citep{tomita1982dynamic} and RNN extraction \citep{weiss2018extracting, wang2018empirical}. Specifically, the benchmark consists of seven regular languages. All languages are defined over the binary alphabet $\Sigma_2 = \{a, b\}$. The languages are numbered $1$-$7$ such that the difficulty of learnability (in an intuitive, informal sense) increases with number. Slight variation exists in the definition of these languages; we use the version reported by \citet{weiss2018extracting}, which is fully documented in \autoref{tab:tomita}.

\begin{table}[!t]
    \centering
    \begin{tabular}{|c|c|}
        \hline
        \# & Definition \\
        \hline
        1 & $a^*$ \\
        2 & $(ab)^*$ \\
        3 & Odd \# of $a$'s must be followed by even \# of $b$'s \\
        4 & All strings without the trigram $aaa$ \\
        5 & Strings $w$ where $\#_a(w)$ and $\#_b(w)$ are even \\
        6 & Strings $w$ where $\#_a(w) \equiv_3 \#_b(w)$ \\
        7 & $b^*a^*b^*a^*$ \\
        \hline
    \end{tabular}
    \caption{Definitions of the Tomita languages. Let $\#_\sigma(w)$ denote the number of occurrences of token $\sigma$ in string $w$. Let $\equiv_3$ denote equivalence mod $3$. $\abs{Q}$ denotes the number of states in the minimum DFA for each language.}
    \label{tab:tomita}
\end{table}

\subsection{Training Details}

To train RNN language recognizers for some formal language $L$, we need data that supervises which strings fall in $L$. Fixing a maximum sequence length $n$, we sample data $\{(x, \mathbf y)\}$, where $x \in \Sigma_2^n$ is a string, and $\mathbf y \in \{0, 1\}^{n+1}$ is a zero-indexed vector of language recognition decisions for each prefix of $n$. For example, given $x=ab$,
\begin{align*}
    y_0 = 1 &\iff \epsilon \in L \\
    y_1 = 1 &\iff a \in L \\
    y_2 = 1 &\iff ab \in L .
\end{align*}
where $\epsilon$ denotes the empty string. To enforce that the dataset is roughly balanced across sequence lengths, we sample half the $x$ uniformly over $\Sigma_2^n$, and, for the other half, enforce that the full string $x$ must be valid in $L$. Given some $x$, the $\mathbf{y}$'s are deterministic to compute. We use a string length of $n=100$ for the training set (with $100,000$ examples), and $n=200$ for a development set (with $1,000$ examples).

We train the RNNs for 22 epochs, choosing the best model by validating with accuracy on a development set.\footnote{As the RNNs converge to 100\% accuracy quickly, we break ties by taking the \emph{highest} epoch to achieve 100\%, which in all cases turns out to be the final epoch.} The architecture consists of an embedding layer (dimension $10$), followed by an RNN layer (dimension $100$), followed by a linear classification head that predicts language membership for each position in the sequence. We use the AdamW optimizer with default hyperparameters.


Before evaluating our method for RNN extraction, we verify that our trained RNNs reach 100\% accuracy. We do this using a held-out generalization set. The sequence length is greater in the generalization set than in training, so strong performance requires generalizing to new lengths. In practice, we find that all RNN language recognizers converge to $100\%$ generalization accuracy within a few epochs.

\section{Extraction Results} \label{sec:eval}

In this section, we evaluate the proposed merging method on how well it describes the behavior of the original RNN. More specifically, we first assess whether the extracted DFA matches the predictions of the RNN on the Tomita languages, and then investigate how different hyperparameters, like the number of data used to built the trie and the dissimilarity tolerance $\kappa$, affect the final output of our algorithm. In summary, we find that our method can extract DFAs matching the original RNN across all 7 languages.


\subsection{Extraction Details}

To build the prefix tree for each language, we use training strings of length 10 where each one is either a member of the language or a random string (with equal probability), and, then, use the trained RNN to compute labels and the representations for each state. We then apply state merging to compress the trie, and the final DFA is evaluated on a held-out set that contains $1, 000$ strings of uniform random length between 0 and 50. Unless otherwise stated, we set $\kappa=0.01$. Finally, a note on the training data: we vary the number of them to evaluate our algorithm's dependence on it, but the number of examples is always orders of magnitude smaller than the number used to train the RNN.

\subsection{Extraction Faithfulness}\label{subsec:extr_faith}

As expected, the merged DFA retains the initial performance over the training set, ensured by the \textbf{Consistency} constraint, while the \textbf{Similarity} one furnishes it with generalization capabilities. As a sanity check, see \autoref{fig:tom2n5_sanity} (left) for the accuracy of the extracted DFA for Tomita 2, one of the ``easy" languages, vs. the accuracy of the initial prefix tree. For Tomita 5, a harder language, the method requires approximately $25$ training strings in the prefix tree to merge down to the perfectly correct DFA (right side of \autoref{fig:tom2n5_sanity}). In general, while the prefix tree may reject previously unseen negative examples, it will never accept previously unseen positive ones. This is the reason why we observe such a large gap in development accuracy between the initial prefix tree and the final merged DFA.

In \autoref{tab:accuracy}, we summarize the results of the extraction on all languages for a fixed number of training data ($n = 300$). We see that in almost all cases the extracted DFA matches the predictions of the RNN. For Tomita 1-6, the algorithm always finds the correct DFA. For Tomita 7, it achieves near 100\% accuracy on every run and recovers the fully correct DFA 3/5 times. We conclude that our algorithm returns a faithful descriptor of the RNN language recognizers for the Tomita languages.
\autoref{tab:accuracy} also compares against a $k$-means baseline \citep{wang2018empirical}.
See \autoref{sec:baseline} for details.
As shown, $k$-means finds the correct DFA for Tomita 1-6, but performs roughly at chance for Tomita 7.

Our results are not directly comparable with \citet{weiss2018extracting} on the same data, as their method learns from active membership and equivalence queries, while ours is designed for the more constrained setting of a static dataset.
However, we note that \citet{weiss2018extracting} were also able to extract faithful DFAs for all $7$ Tomita languages using their $L^\star$ method.
Unlike $L^*$, however, our method does not make use of potentially expensive equivalence queries.

\begin{table}[!t]
    \centering
    \begin{tabular}{|c|cc|cc|c|}
        \hline
        & \multicolumn{2}{|c|}{State merging} & \multicolumn{2}{|c|}{$k$-means} & \\
        \# & Acc & $\abs{\hat Q}$ & Acc & $\abs{\hat Q}$ & $\abs{Q}$ \\
        \hline
        1 & 100. $\pm$ 0. & 1 & 100. $\pm$ 0. & 1 & 1 \\
        2 & 100. $\pm$ 0. & 2 & 100. $\pm$ 0. & 2 & 2 \\
        3 & 100. $\pm$ 0. & 4 & 100. $\pm$ 0. & 4 & 4 \\
        4 & 100. $\pm$ 0. & 3 & 100. $\pm$ 0. & 3 & 3 \\
        5 & 100. $\pm$ 0. & 4 & 100. $\pm$ 0. & 4 & 4 \\
        6 & 100. $\pm$ 0. & 3 & 100. $\pm$ 0. & 3 & 3 \\
        7 & {\bf 99.62 $\pm$ 0.55} & {\bf 4} & 57.35 $\pm$ 0.25 & 1 & 4 \\
        \hline
    \end{tabular}
    \caption{Mean accuracy together with standard deviation of the extracted DFA on the 7 Tomita languages. Randomness induced by 5 random seeds for sampling data to build the prefix tree. ``$\abs{\hat Q}$'' is the smallest extracted DFA size after minimization. ``$\abs{Q}$'' reports the size of the true minimum DFA for each language. State merging is our method; $k$-means is a baseline based on \citet{wang2018empirical}.}
    \label{tab:accuracy}
\end{table}

\begin{figure*}[!ht]
    \centering
    \includegraphics[width=0.49\linewidth]{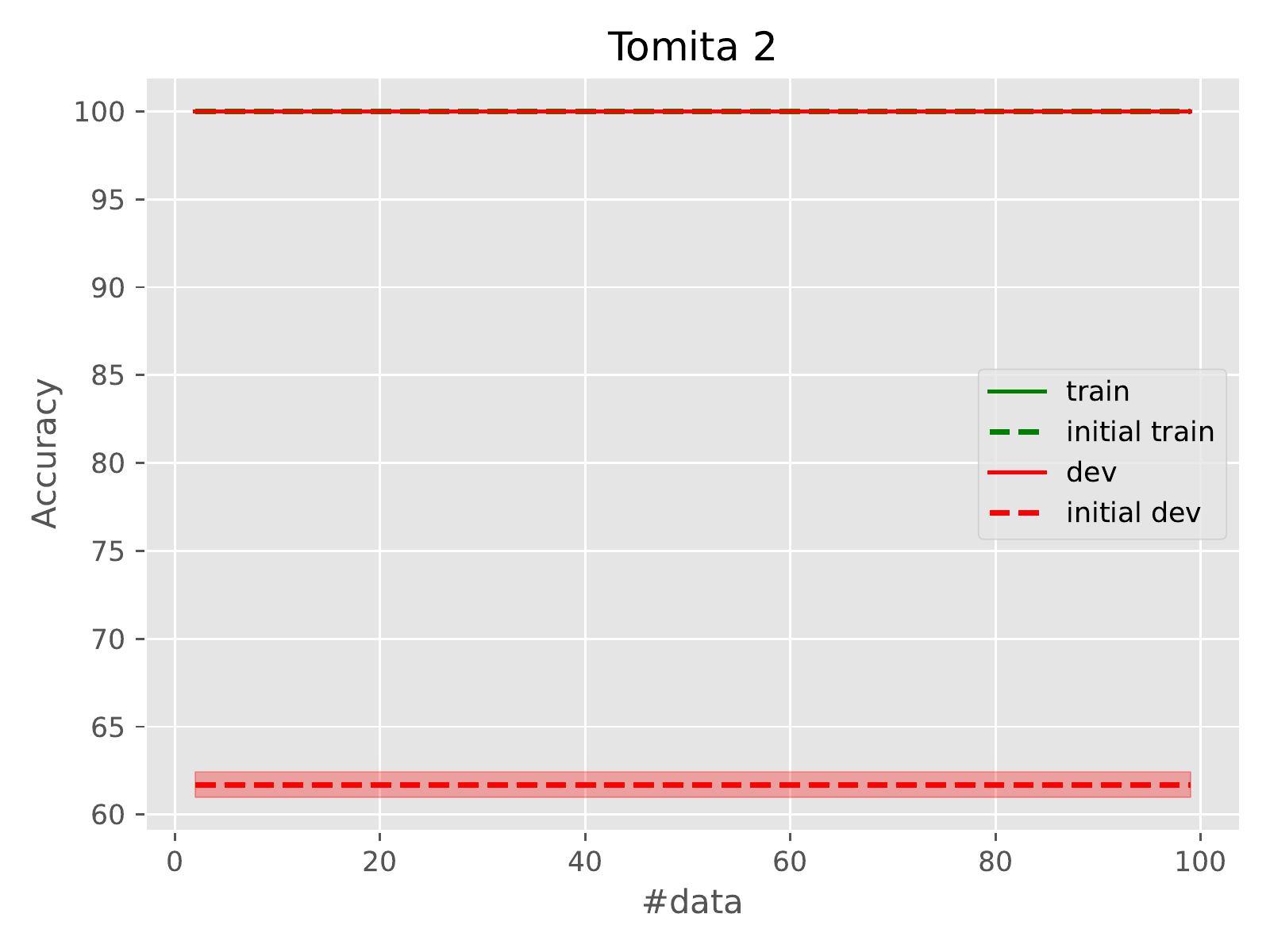}
    \includegraphics[width=0.49\linewidth]{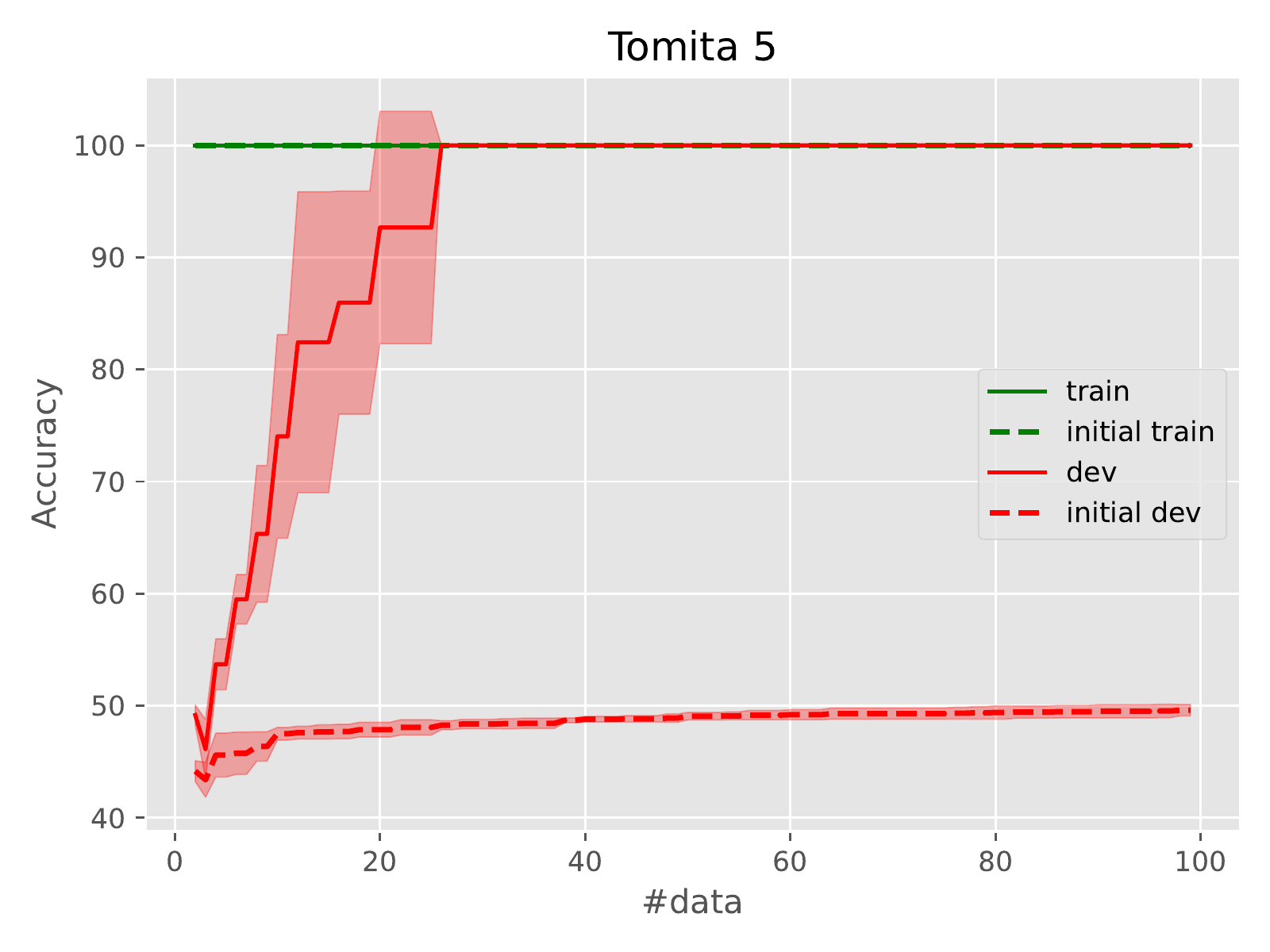}
    \caption{Faithfulness of extracted DFA. For Tomita 2 (left) and Tomita 5 (right), we extract DFAs that (i) are consistent with the training set, and (ii) reach 100$\%$ dev accuracy. Notice that the initial prefix tree records trivial accuracy on the dev set (red line). Trend line shows the average across $3$ random seeds (different datasets), and the shaded region denotes one std deviation.}
    \label{fig:tom2n5_sanity}
\end{figure*}

\subsection{Effect of Similarity Threshold}

Now, we assess qualitatively the role of the similarity threshold $\kappa$ on the output of the state merging algorithm. We use as case study Tomita 2, whose recognizing automata are easy to visualize and interpret.
As we see in Figure \ref{fig:dfa_img}, the choice of $\kappa$ affects crucially the output of the algorithm. With large tolerance $\kappa = .5$, we ``overmerge" states resulting with a trivial 2-state DFA, that accepts only the empty string. Gradually decreasing $\kappa$ produces the desirable effect. For instance, for $\kappa = .4$ we find an almost minimal DFA that describes our language.
Finally, for a very strict threshold $1 - 0.99$ that decides only to merge states whose $100$-dimensional representations are very well aligned, we recover a correct, but highly redundant DFA. Applying minimization to this DFA produces the correct $2$-state DFA. The connection between the number of states in the unminimized DFA and the quality of the representations afforded by the RNN is further discussed in Section \ref{sec:beyond_conv}.

\begin{figure}
    \begin{subfigure}[b]{0.99\columnwidth}
        \centering
            \resizebox{0.99\columnwidth}{!}{\begin{tikzpicture}[->,>=stealth',shorten >=1pt,auto,node distance=2cm, semithick]
        \node[initial,state,accepting] (0) {$q_0$};
        \node[state,above right of=0] (1) {$q_1$};
        \node[state,below right of=0] (2) {$q_2$};
        \node[state,accepting,right of=1] (3) {$q_3$};
        \node[state,right of=3] (4) {$q_4$};
        \node[state,accepting,above right of=4] (5) {$q_5$};
        \node[state,right of=5] (7) {$q_7$};
        \node[state,accepting,right of=7] (8) {$q_8$};
        \node[state,right of=8] (9) {$q_9$};
        \node[state,accepting,right of=9] (10) {$q_{10}$};
        \node[state,right of=10] (11) {$q_{11}$};
        \node[state,accepting,right of=11] (12) {$q_{12}$};
        \node[state,below right of=4] (6) {$q_6$};
        \node[state,right of=6] (22) {$q_{22}$};
        \node[state,right of=22] (23) {$q_{23}$};
        \node[state,right of=23] (24) {$q_{24}$};
        \node[state,right of=24] (25) {$q_{25}$};
        \node[state,right of=25] (26) {$q_{26}$};
        \node[state,right of=26] (27) {$q_{27}$};
        \node[state,right of=2] (13) {$q_{13}$};
        \node[state,right of=13] (14) {$q_{14}$};
        \node[state,right of=14] (15) {$q_{15}$};
        \node[state,right of=15] (16) {$q_{16}$};
        \node[state,right of=16] (17) {$q_{17}$};
        \node[state,right of=17] (18) {$q_{18}$};
        \node[state,right of=18] (19) {$q_{19}$};
        \node[state,right of=19] (20) {$q_{20}$};
        \node[state,right of=20] (21) {$q_{21}$};
        
        \path (0) edge [above] node {$a$} (1);
        \path (0) edge [above] node {$b$} (2);
        \path (1) edge [above] node {$b$} (3);
        \path (3) edge [above] node {$a$} (4);
        \path (4) edge [above] node {$b$} (5);
        \path (5) edge [above] node {$a$} (7);
        \path (7) edge [above] node {$b$} (8);
        \path (8) edge [above] node {$a$} (9);
        \path (9) edge [above] node {$b$} (10);
        \path (10) edge [above] node {$a$} (11);
        \path (11) edge [above] node {$b$} (12);
        \path (4) edge [above] node {$a$} (6);
        \path (6) edge [above] node {$b$} (22);
        \path (22) edge [above] node {$a$} (23);
        \path (23) edge [above] node {$b$} (24);
        \path (24) edge [above] node {$a$} (25);
        \path (25) edge [above] node {$a$} (26);
        \path (26) edge [above] node {$b$} (27);
        \path (2) edge [above] node {$b$} (13);
        \path (13) edge [above] node {$a$} (14);
        \path (14) edge [above] node {$b$} (15);
        \path (15) edge [above] node {$b$} (16);
        \path (16) edge [above] node {$b$} (17);
        \path (17) edge [above] node {$b$} (18);
        \path (18) edge [above] node {$b$} (19);
        \path (19) edge [above] node {$b$} (20);
        \path (20) edge [above] node {$a$} (21);
    \end{tikzpicture}}
    \caption{Prefix tree}
    \end{subfigure}
    \begin{subfigure}[b]{0.45\columnwidth}
        \centering
        \resizebox{0.99\columnwidth}{!}{\begin{tikzpicture}[->,>=stealth',shorten >=1pt,auto,node distance=2cm, semithick]
        \node[initial,state,accepting] (2) {$q_0$};
        \node[state,right of=2] (3) {$q_1$};
        
        \path (2) edge [above] node {$a, b$} (3);
        \path (3) edge [loop above] node {$a, b$} (3);
    \end{tikzpicture}}
    \caption{$\kappa=0.5$}
    \end{subfigure}
    \begin{subfigure}[b]{0.45\columnwidth}
        \centering
        \resizebox{0.99\columnwidth}{!}{\begin{tikzpicture}[->,>=stealth',shorten >=1pt,auto,node distance=2cm, semithick]
        \node[initial,state,accepting] (0) {$q_0$};
        \node[state,right of=0] (1) {$q_1$};
        \node[state,below of=0] (2) {$q_2$};
        \node[state,right of=2] (3) {$q_3$};
        
        \path (0) edge [bend left, above] node {$a$} (1);
        \path (1) edge [bend left, below] node {$b$} (0);
        \path (0) edge [left] node {$b$} (2);
        \path (2) edge [above] node {$b$} (3);
        \path (3) edge [loop right] node {$a, b$} (3);
    \end{tikzpicture}}
    \caption{$\kappa=0.4$}
    \end{subfigure}
    \begin{subfigure}[b]{0.9\columnwidth}
        \centering
        \resizebox{0.99\columnwidth}{!}{\begin{tikzpicture}[->,>=stealth',shorten >=1pt,auto,node distance=2cm, semithick]
        \node[initial,state,accepting] (0) {$q_0$};
        \node[state,above right of=0] (1) {$q_1$};
        \node[state,below right of=0] (2) {$q_2$};
        \node[state,accepting,right of=1] (3) {$q_3$};
        \node[state,right of=3] (4) {$q_4$};
        \node[state,accepting,right of=4] (5) {$q_5$};
        \node[state,right of=2] (6) {$q_6$};
        \node[state,right of=6] (7) {$q_7$};
        \node[state,right of=7] (8) {$q_8$};
        \node[state, above right of=8] (10) {$q_{10}$};
        \node[state, below right of=8] (9) {$q_{9}$};
        
        \path (0) edge [above] node {$a$} (1);
        \path (1) edge [above] node {$b$} (3);
        \path (3) edge [above] node {$a$} (4);
        \path (4) edge [bend left, above] node {$b$} (5);
        \path (5) edge [bend left, above] node {$a$} (4);
        \path (0) edge [above] node {$b$} (2);
        \path (2) edge [above] node {$b$} (6);
        \path (6) edge [above] node {$a$} (7);
        \path (7) edge [above] node {$b$} (8);
        \path (7) edge [bend right, below] node {$a$} (9);
        \path (8) edge [loop above] node {$b$} (8);
        \path (8) edge [above] node {$a$} (10);
        \path (10) edge [right] node {$a$} (9);
        \path (9) edge [loop below] node {$a$} (9);
        \path (9) edge [above] node {$b$} (8);
    \end{tikzpicture}}    
    \caption{$\kappa=0.01$}
    \end{subfigure}
    \caption{Initial prefix tree and resulting merged DFAs for different values of $\kappa$. Language: Tomita 2.}
    \label{fig:dfa_img}
\end{figure}
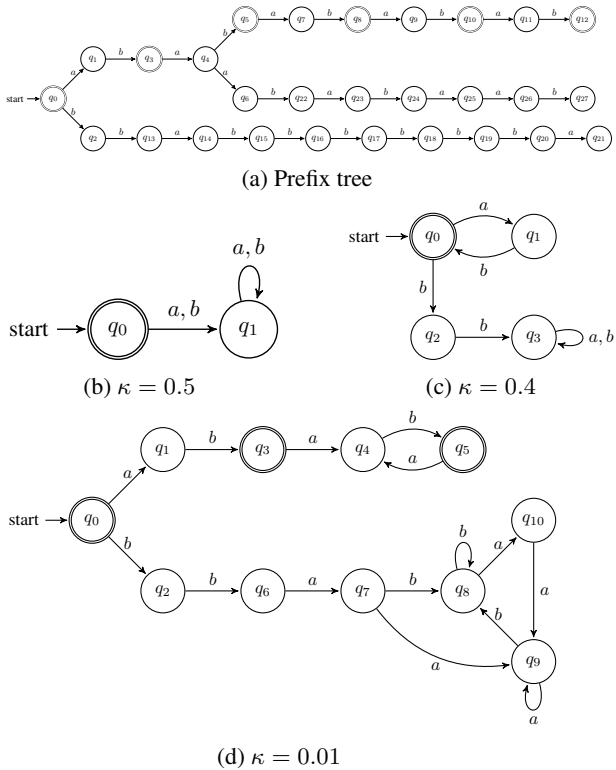

\begin{figure*}[!ht]
    \centering
    \includegraphics[width=.48\linewidth]{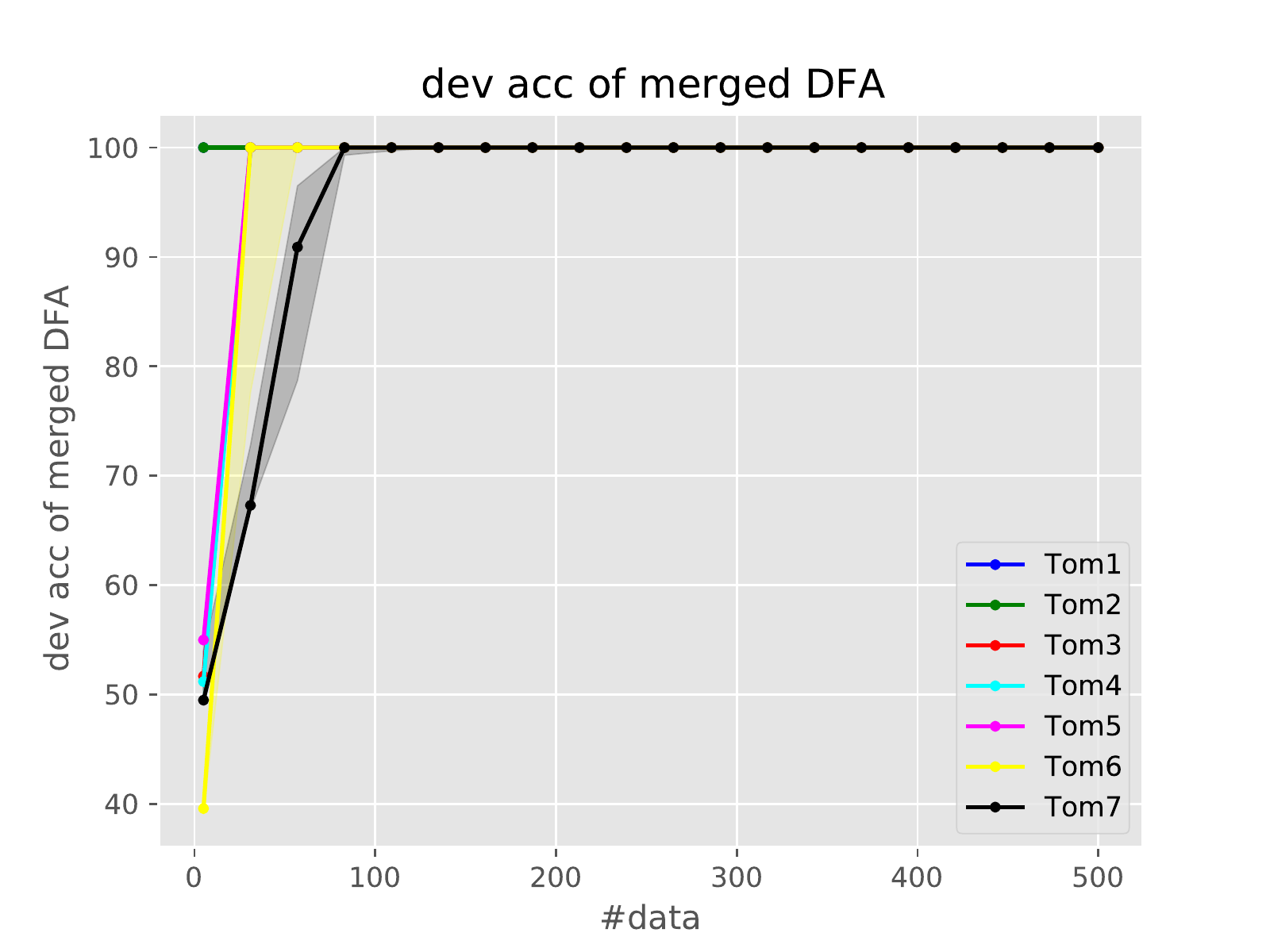}
    \includegraphics[width=.48\linewidth]{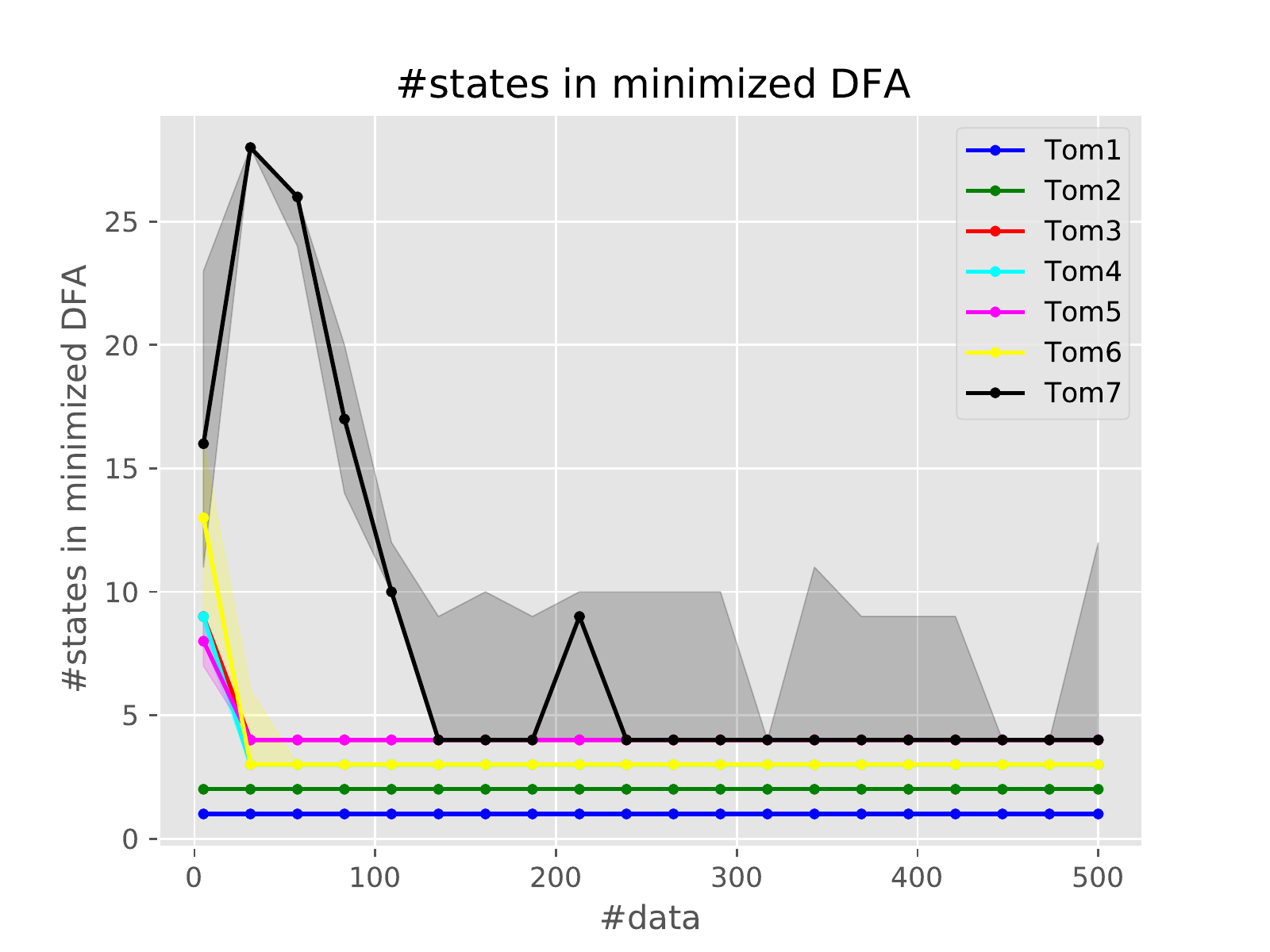}
    \caption{Left: The DFA's accuracy at reconstructing the RNN on the development set as the number of data used to build the prefix tree increases. Right: the number of states in the minimized DFA obtained through state merging, which plateaus for Tomita 1-6, and sometimes reaches the ideal value for Tomita 7. Trend line shows the median across $5$ random seeds (different datasets), and the $0.25$--$0.75$ percentile region is shaded. The prefix tree is built from sentences of length $15$ here.}
    \label{fig:by-data}
\end{figure*}

\subsection{Effect of Data Size}

Across all 7 Tomita languages, using more data to build the prefix tree improves our ability to extract the correct DFA. As seen in \autoref{fig:by-data}, our method reaches $100\%$ accuracy at matching the predictions of the RNN for all languages, although Tomita 7 has high variance even at the end. On the right side of \autoref{fig:by-data}, we show that the number of states in the final DFA decreases with the number of data, reaching the true minimum DFA for Tomita 1-6. For Tomita 7, we first recover the minimum DFA when using $135$ data points, but state merging does not always produce the true minimum DFA, even at 500 data points.

\subsection{Effect of Training Beyond Convergence}\label{sec:beyond_conv}

\begin{figure*}[!t]
    \centering
    \includegraphics[width=.48\linewidth]{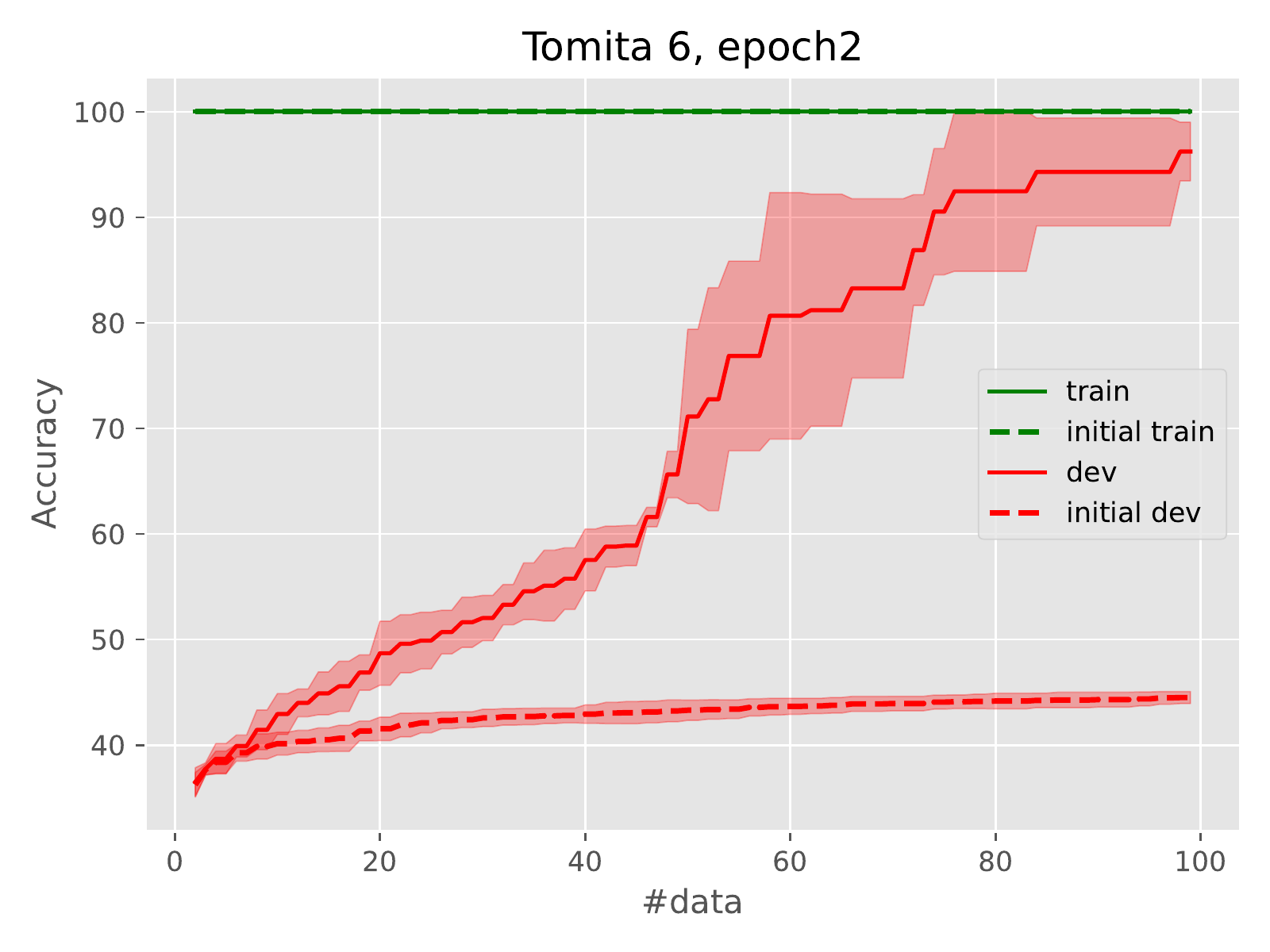}
    \includegraphics[width=0.48\linewidth]{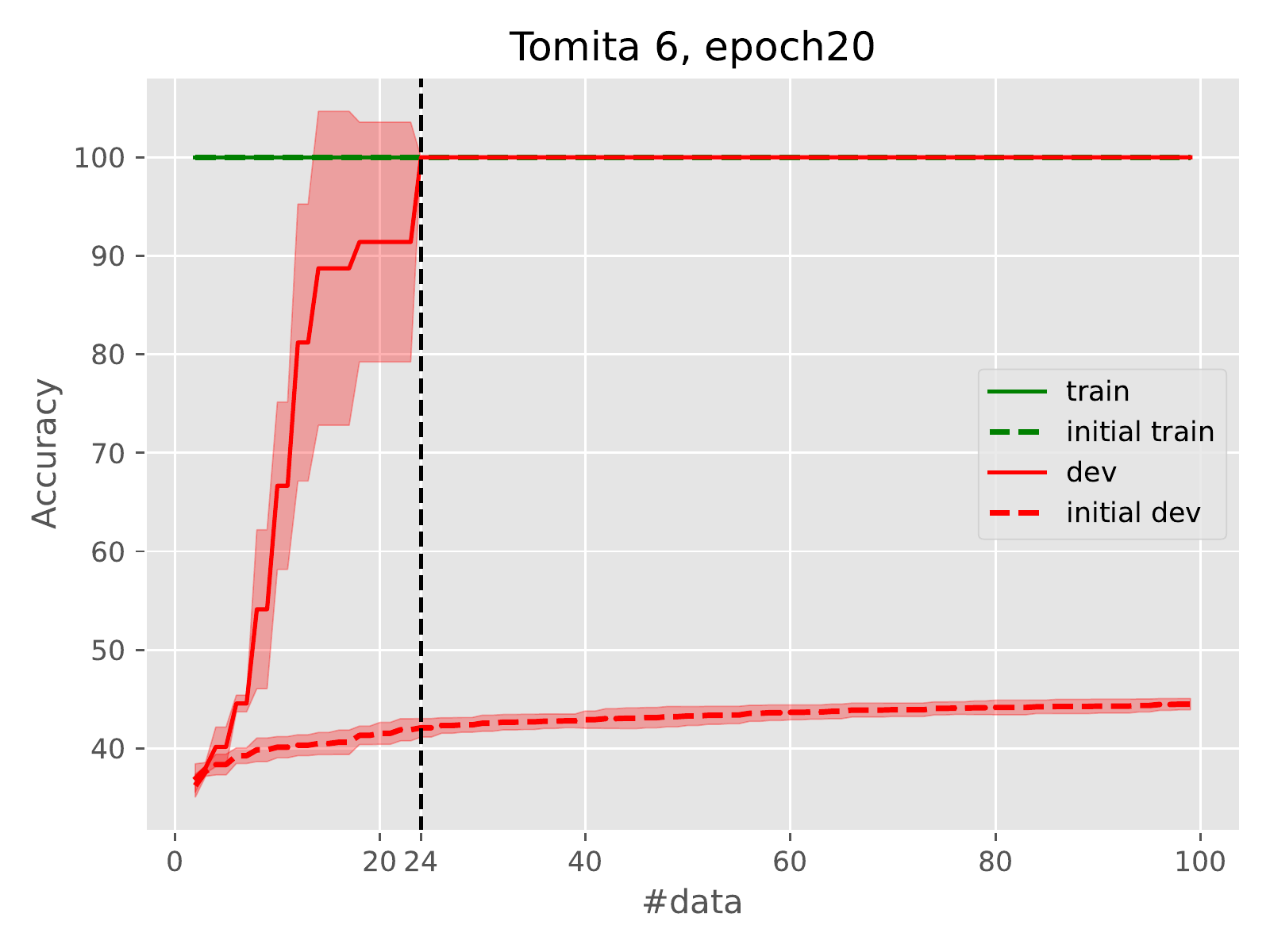}
    \caption{Accuracy of extracted DFA vs number of data for two different sets of representations for Tomita 6. Left: Representations from an RNN trained for 2 epochs. Right: Representations from an RNN trained for 20 epochs. The RNNs record the same dev accuracy, but the automata extracted from the longer-trained RNN has higher accuracy. Trend line shows the average across $3$ random seeds (different datasets), and the shaded region denotes one std deviation. $\kappa$ was set equal to .99.}
    \label{fig:tom6_dif_ep}
\end{figure*}

In all cases, our RNNs converged to 100\% development accuracy within one or two epochs. However, we find that continuing to train beyond this point can improve the ability of our method to extract the true gold-standard DFA.

First, we investigate the effect of continued training on learning Tomita 6, one of the ``difficult" Tomita languages, using representations extracted from an RNN after 2 and 20 epochs of training respectively. The results are shown in \autoref{fig:tom6_dif_ep}. Although the RNN development accuracy is 100\% at both checkpoints, the extraction results are now different; as the figure illustrates, a correct automaton can be extracted from the ``overly" trained network with much fewer data than its less trained ancestor. This suggests that additional training is somehow improving or simplifying the quality of the representations, even after development accuracy has converged to 100\%. The behavior across all 7 languages is similar; see \autoref{app:extras}.


\begin{figure*}[!t]
    \centering
    \includegraphics[width=.48\linewidth]{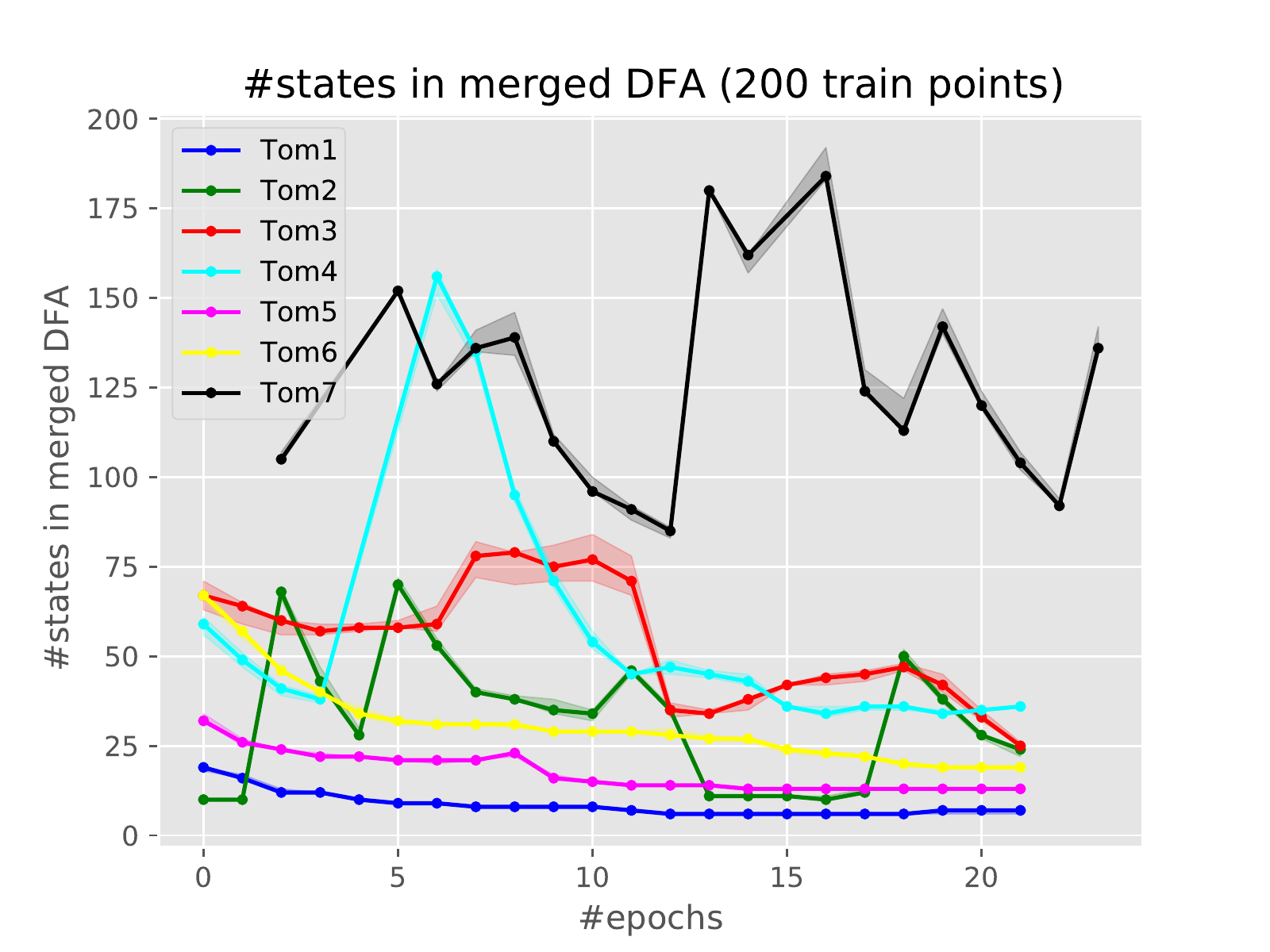}
    \includegraphics[width=.48\linewidth]{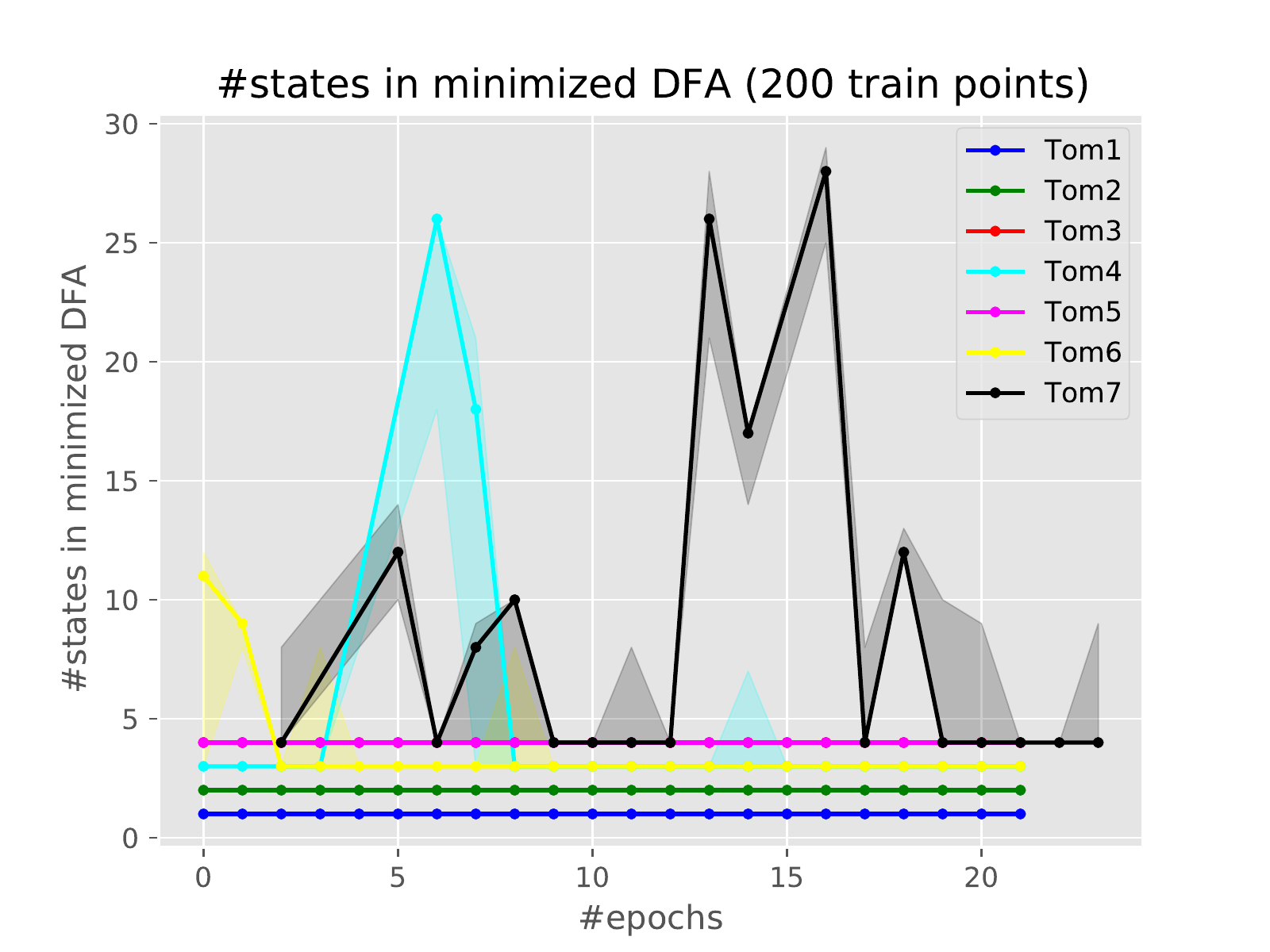}
    \caption{Number of states in the DFA achieved through state merging (left) and the minimized version of that DFA (right) training beyond convergence. As an artifact of our training logic, Tomita 7 was trained for 24 epochs, and epochs 0-1 were lost. The merged DFA size tend to gradually decrease for Tomita 1-6 with more training epochs, suggesting implicit merging during training. The minimized DFA reliably reaches the theoretical minimum DFA for all languages besides Tomita 7, \new{for which it finds the correct minimum DFA 3/5 times}.}
    \label{fig:n_states}
\end{figure*}

Next, we study the effect of continued training on the complexity of the extracted automaton, measured in the number of states. As shown in the left side of \autoref{fig:n_states}, the number of states in DFA obtained through state merging (but before minimization) tends to decrease gradually with more additional training for Tomita 1-6, despite some rapid upward spikes. This suggests that additional training is perhaps simplifying the structure of the RNN's state space by merging redundant states together. On the right, we can see that for Tomita 1-6, training for enough additional epochs brings the size of the minimized DFA down to the ideal minimum DFA size. Together, these results suggest that additional training may be simplifying the implicit RNN state space to remove redundant state representations, thereby improving our ability to extract the true minimum DFA for the language it recognizes. We call this phenomenon \textit{implicit merging} induced by the training procedure.


\paragraph{Speculative Explanations} Why should training beyond convergence lead to easier extraction? One potential explanation is the ``saturation'' phenomenon of neural net training \citep{karpathy2015visualizing, merrill2020parameter}: if training consistency increases the parameter $2$-norm (which we find to hold for our RNNs), then training for more time should lead the RNN to more closely approximate infinite-norm RNNs. Infinite-norm RNNs can be directly viewed as DFAs \citep{merrill-2019-sequential}, which may explain why it is easier to extract a DFA from a network trained significantly beyond convergence.
Implicit merging is also consistent with the information bottleneck theory of deep learning \cite{ZiTi17} that identifies two phases during training; first, a data-fitting period and then a compressing one. Interestingly, it has been observed that it is the saturating nature of non-linearities that yields the compression phase \cite{Sax+18}, further supporting our previous explanation.
More speculatively, it is possible that the benefit of training beyond convergence for extraction could be related to ``grokking'' \citep{power2021grokking}: a phenomenon where generalization on synthetic formal language tasks begins to improve after continuing to train for hundreds of epochs past convergence in training accuracy. Along these lines, it would be interesting for future work to continue investigating the mechanism through which training beyond convergence can improve the ease of RNN extraction, as this may provide interesting insight into the implicit regularization that RNNs receive during training.

\section{Conclusion} \label{sec:conclusion}

We have shown how state merging can be used to extract automata from RNNs that capture both the decision behavior and representational structure of the original blackbox RNN. Using state merging, we were able to extract faithful automata from RNNs trained on all 7 Tomita languages, demonstrating the effectiveness of our method. For future work, it would be useful to find ways to extend the state merging extraction algorithm to scale to larger state spaces and alphabet sizes. One interesting empirical finding is that continuing to train an RNN \emph{after} it has perfectly learned the target language improves the sample efficiency of extracting an automaton from it. Our analysis of this \emph{implicit merging} phenomenon suggests that training past convergence may lead to a more robust representation of the underlying state space within the RNN through an implicit regularizing effect where neighborhoods representing the same state converge to single vector representations. Under this view, gradient-descent-like training itself may be viewed as a state merging process: training may exude a pressure to compress similar states together, producing a simpler model that may generalize better to unseen strings.


\bibliography{references.bib}

\newpage
\appendix
\onecolumn

\section{Baseline Method} \label{sec:baseline}

We describe the details of the $k$-means extraction method that we use as a baseline \citep{wang2018empirical}.

Taking a train set of the same form as state merging, we collect all the hidden states from every prefix of every train string, each of which is associated with a label, i.e., whether the prefix is in $L$. This yields a dataset of the form $\{(\mathbf h_{ij}, y_{ij})\}$, where $\mathbf h_{ij}$ is the RNN hidden state on word $j$ of example $i$, and $y_{ij}$ records whether $w_{i,:j} \in L$.

We apply $k$-means clusters to the hidden states, where $k$ is a hyperparameter that we set to $20$. Manual inspection reveals that the results are not particularly sensitive to $k$, which makes sense given the fact that all Tomita languages require at most $7$ states. We identify each cluster with a state in that DFA that will be extracted. We then need to decide which cluster is the initial state, and, for each cluster, whether it is accepting, and to which clusters it transitions for each input token. We find the initial state by checking which cluster is assigned to the \texttt{<bos>} symbol by the RNN. We compute whether a cluster is accepting or rejecting by taking a majority vote for each $y_{ij}$ in the cluster. Finally, for a token $\sigma$, we assign the transition out of a cluster by collecting all hidden states that are achieved after observing $\sigma$ in that cluster, finding the corresponding clusters, and taking a majority vote.

\section{Missing Plots} \label{app:extras}
\begin{figure*}[!ht]
    \centering
    \includegraphics[width=0.24\linewidth]{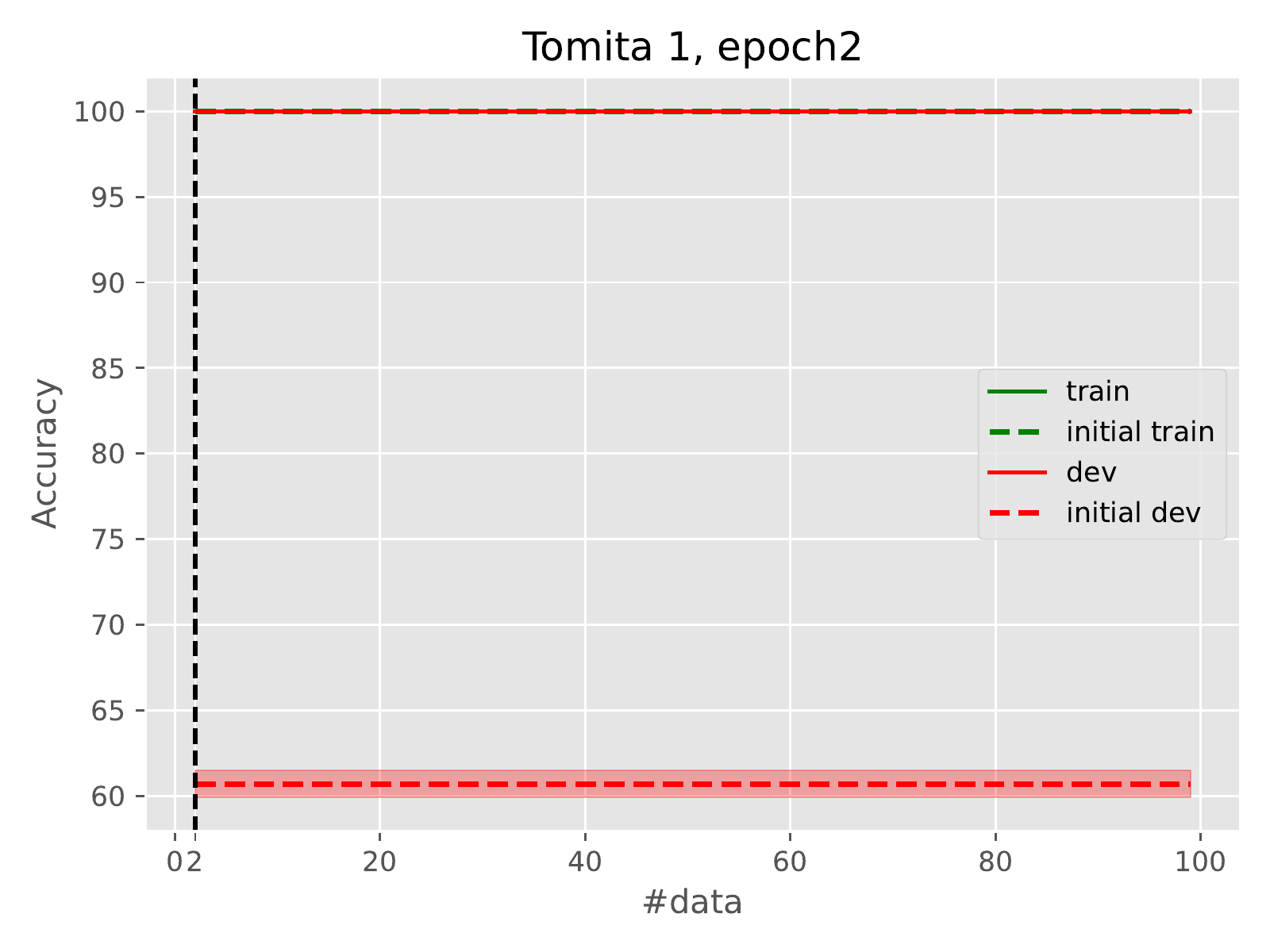}
    \includegraphics[width=0.24\linewidth]{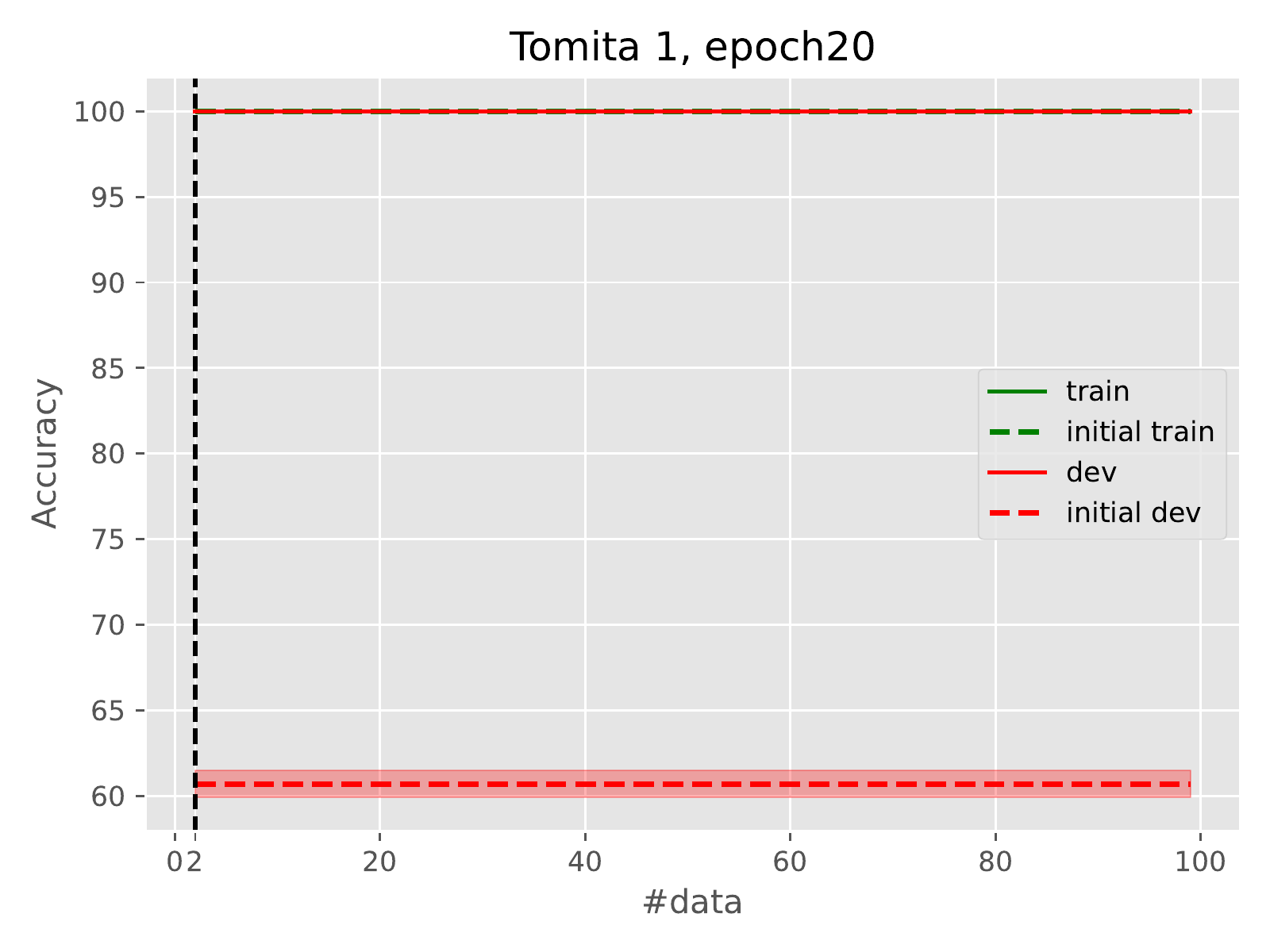}
    \includegraphics[width=0.24\linewidth]{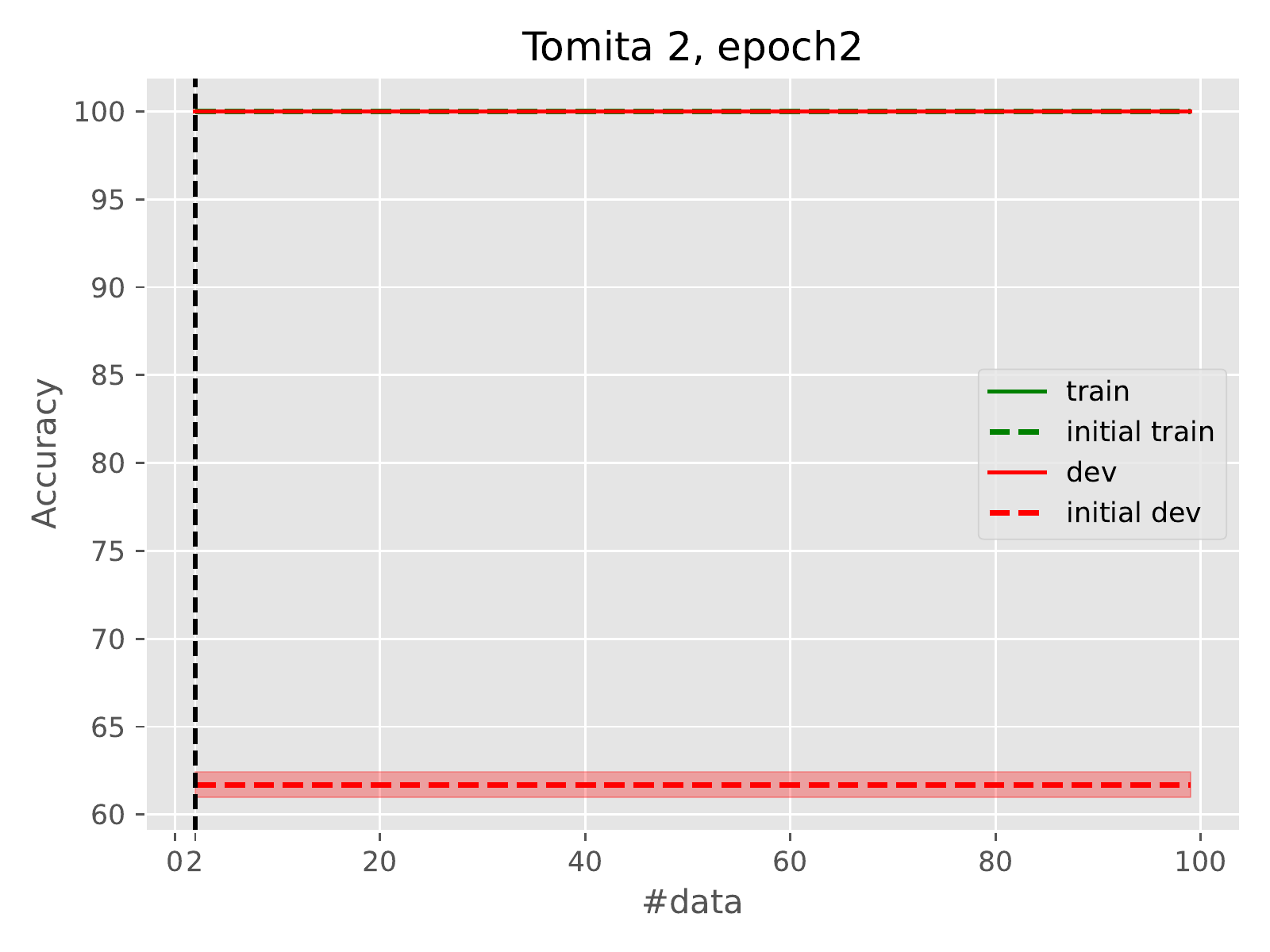}
    \includegraphics[width=0.24\linewidth]{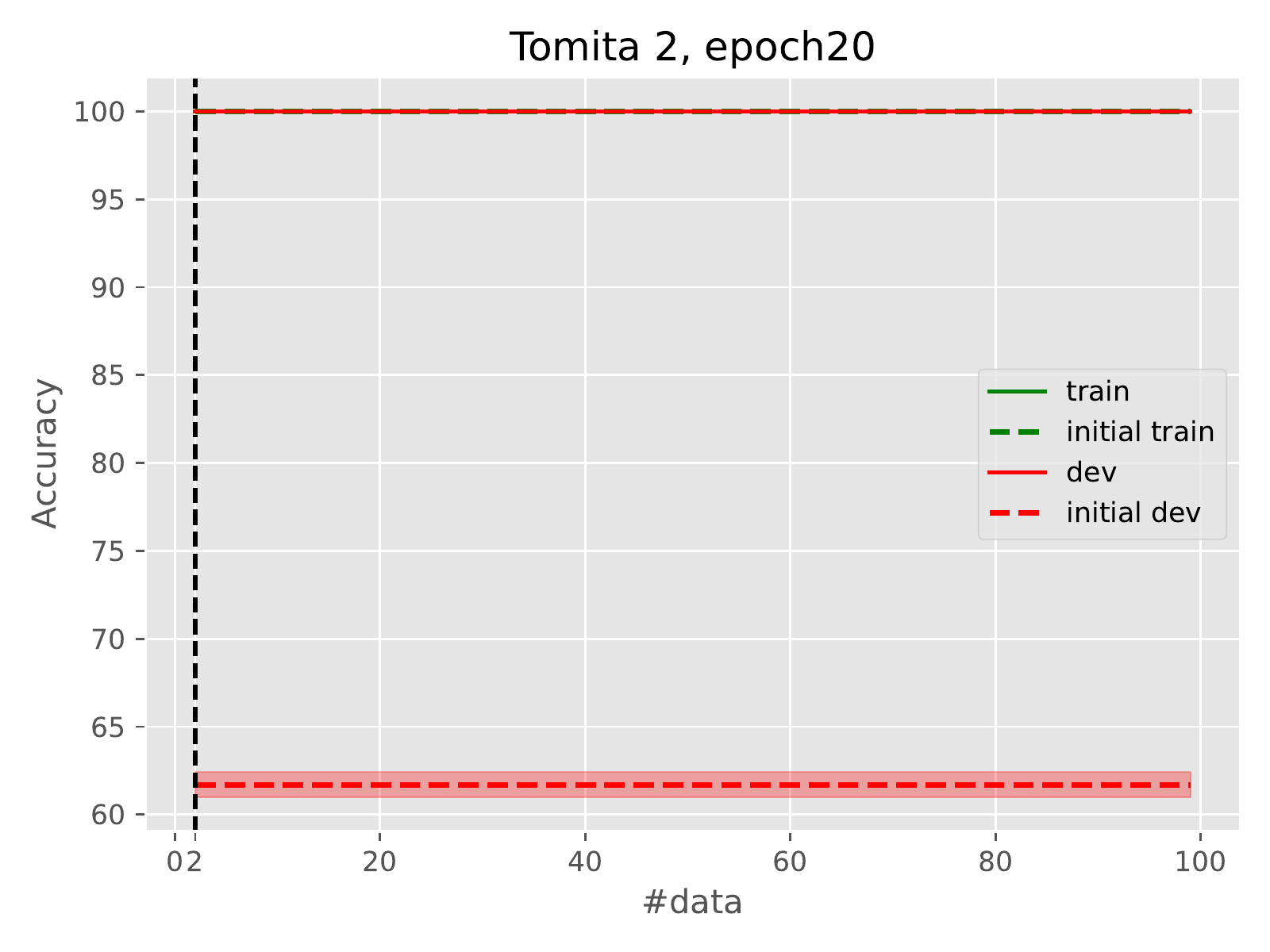}
    \includegraphics[width=0.24\linewidth]{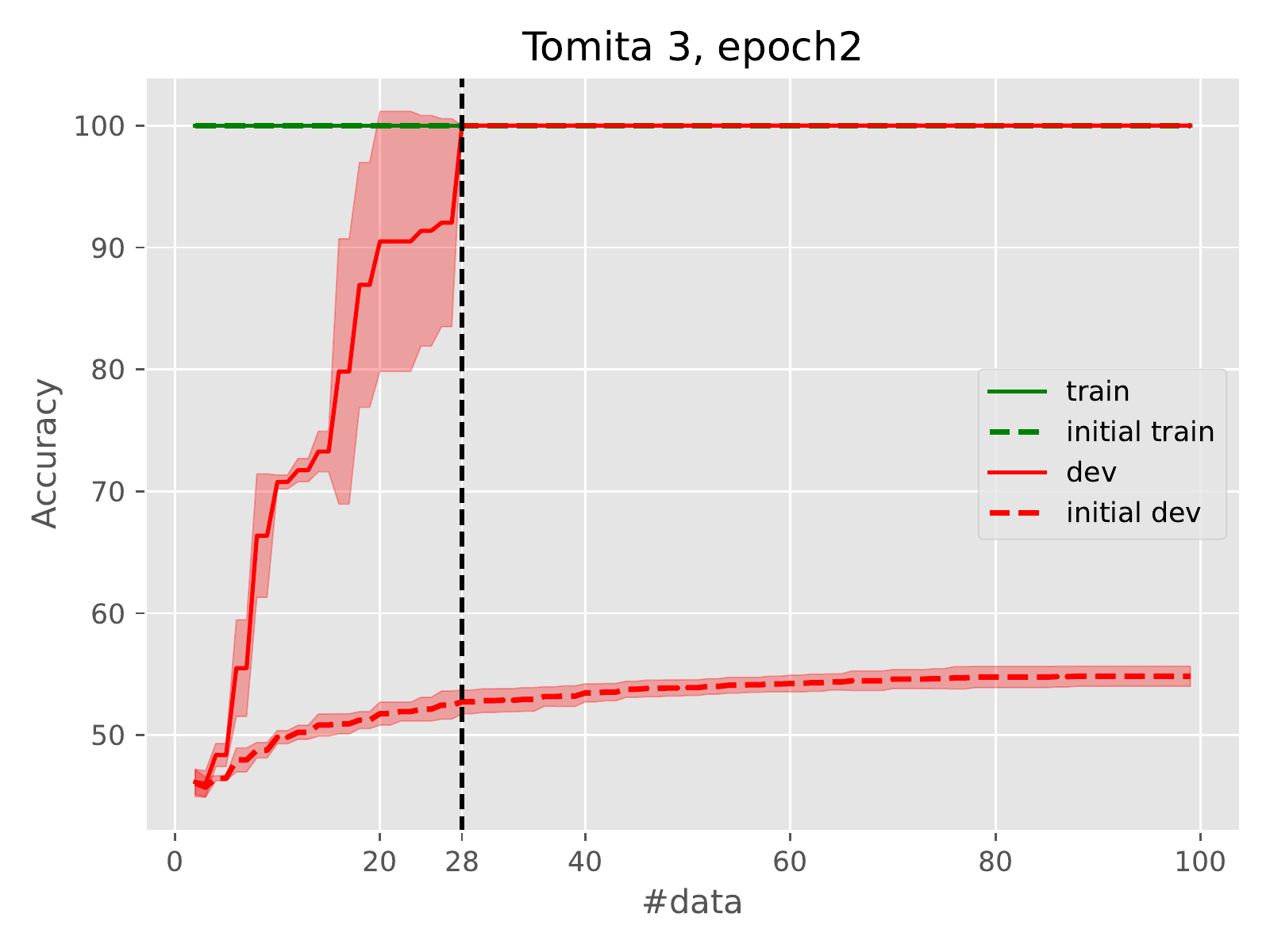}
    \includegraphics[width=0.24\linewidth]{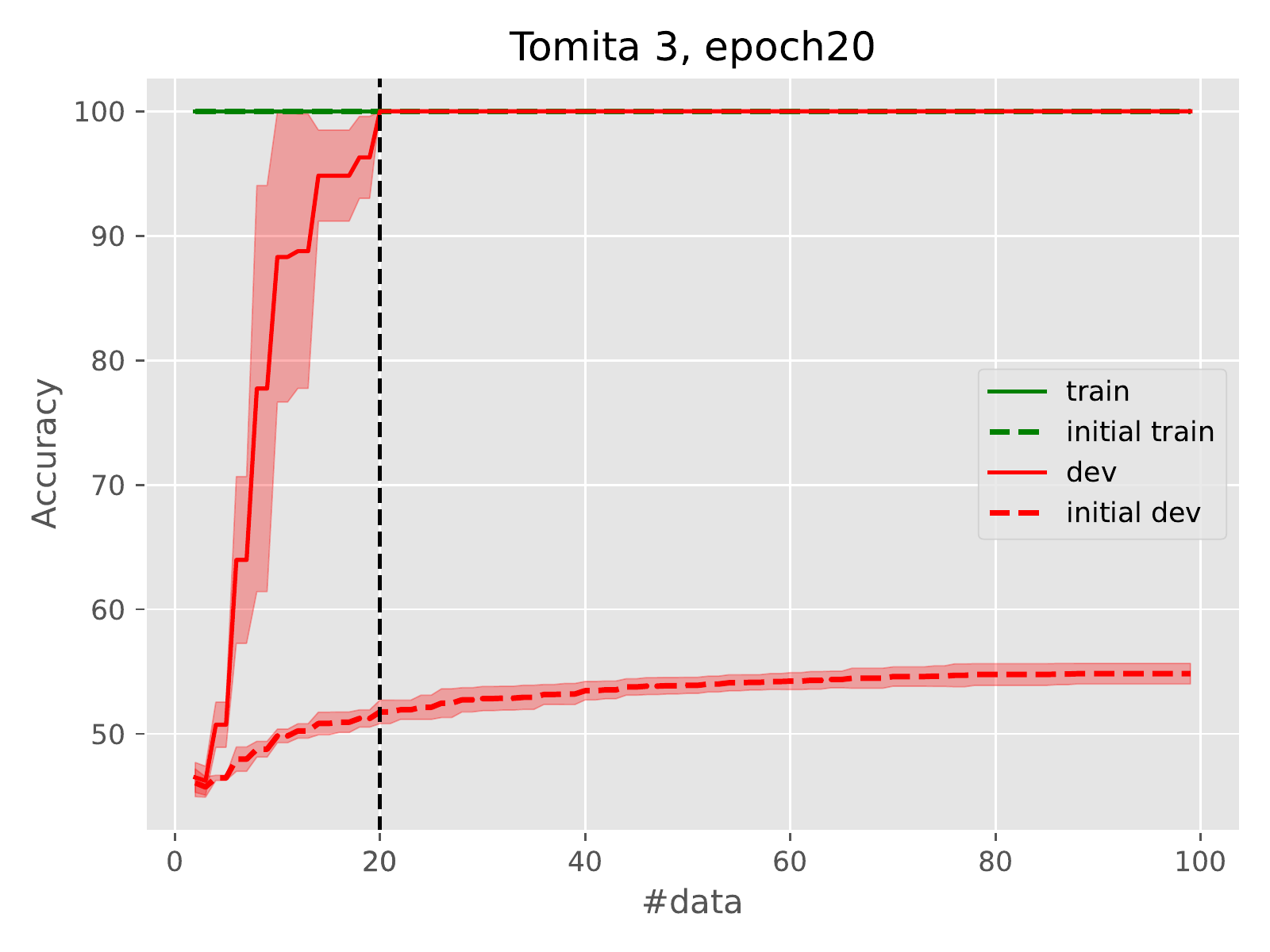}
    \includegraphics[width=0.24\linewidth]{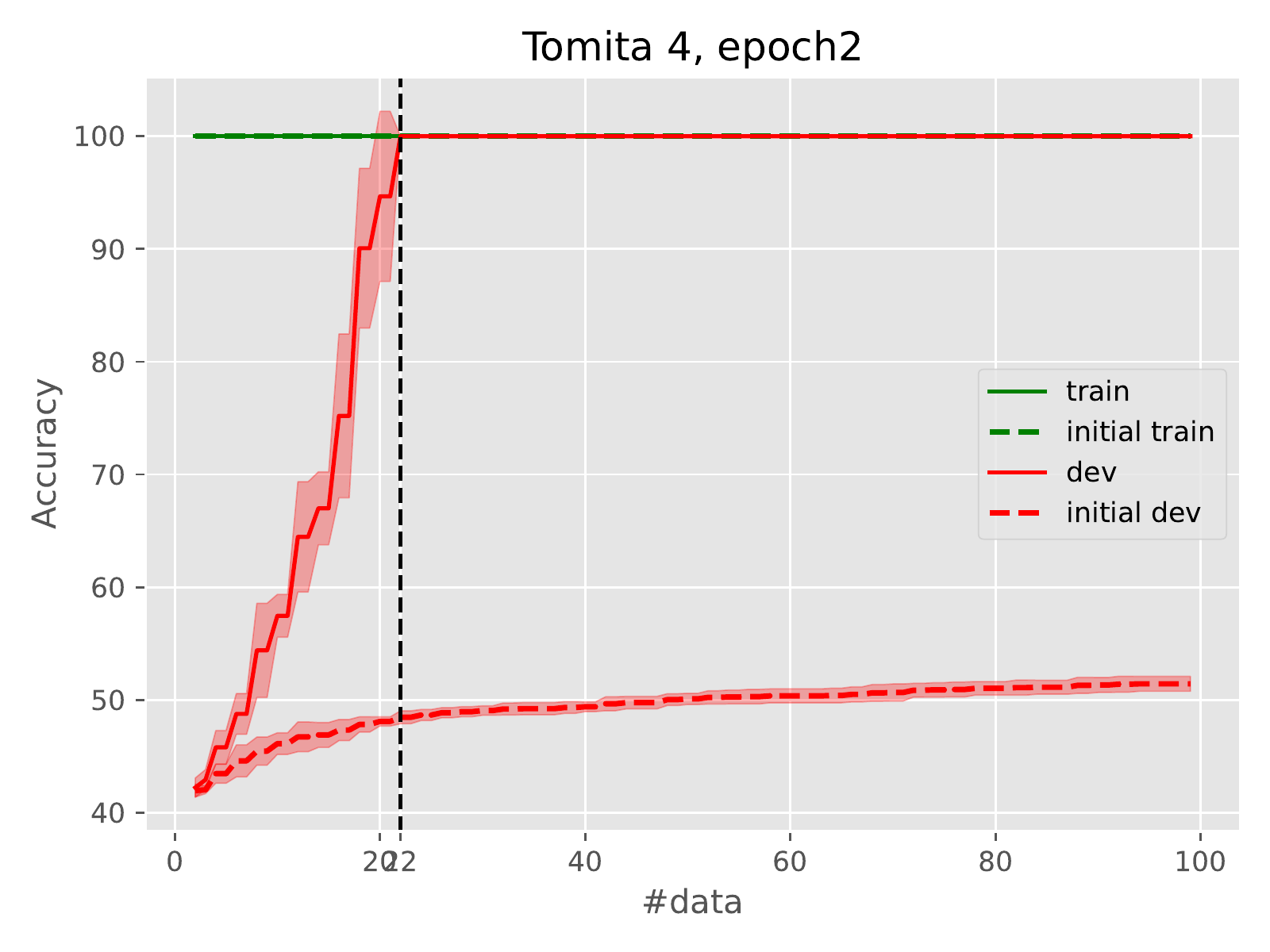}
    \includegraphics[width=0.24\linewidth]{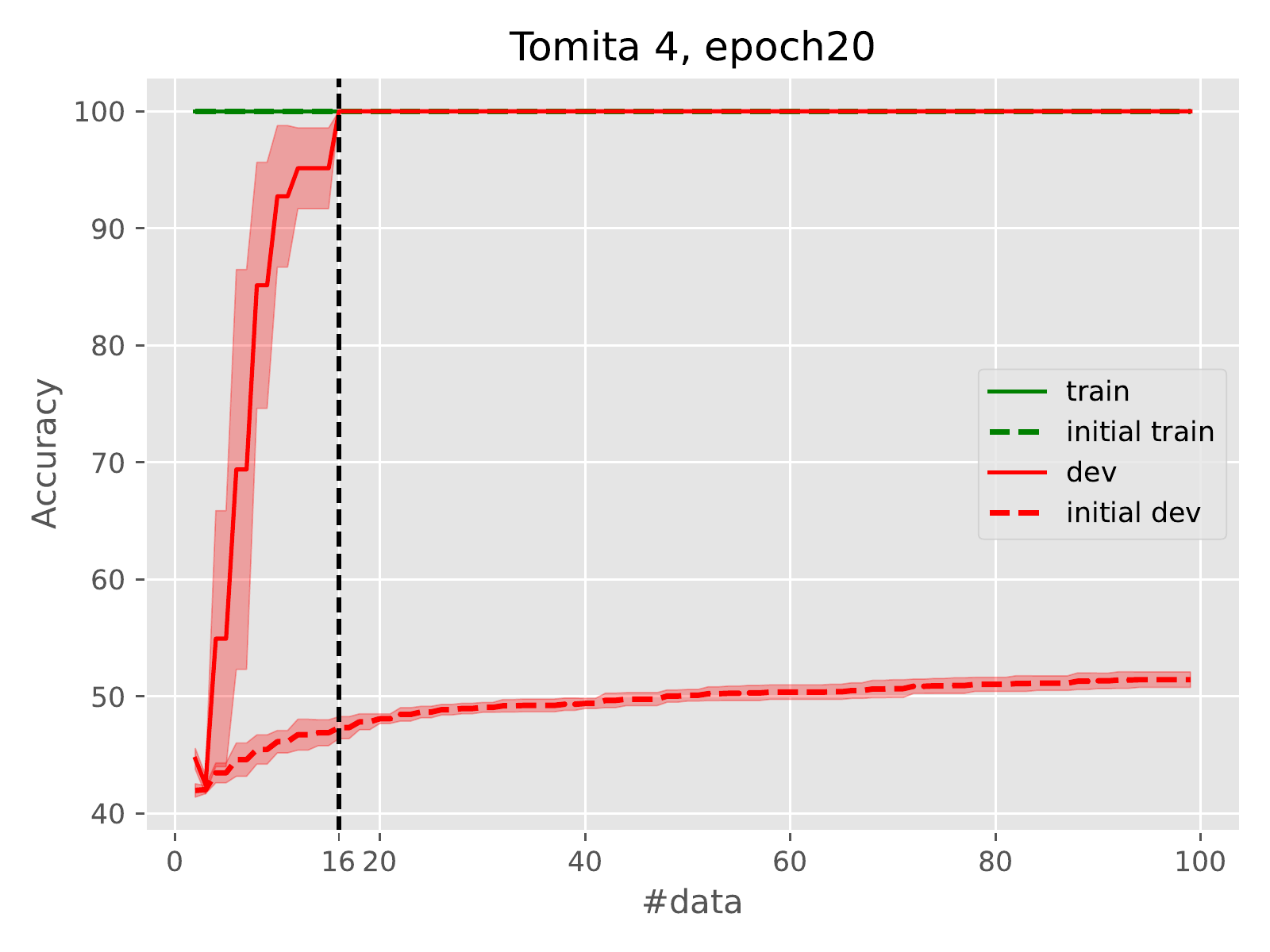}
    \includegraphics[width=0.24\linewidth]{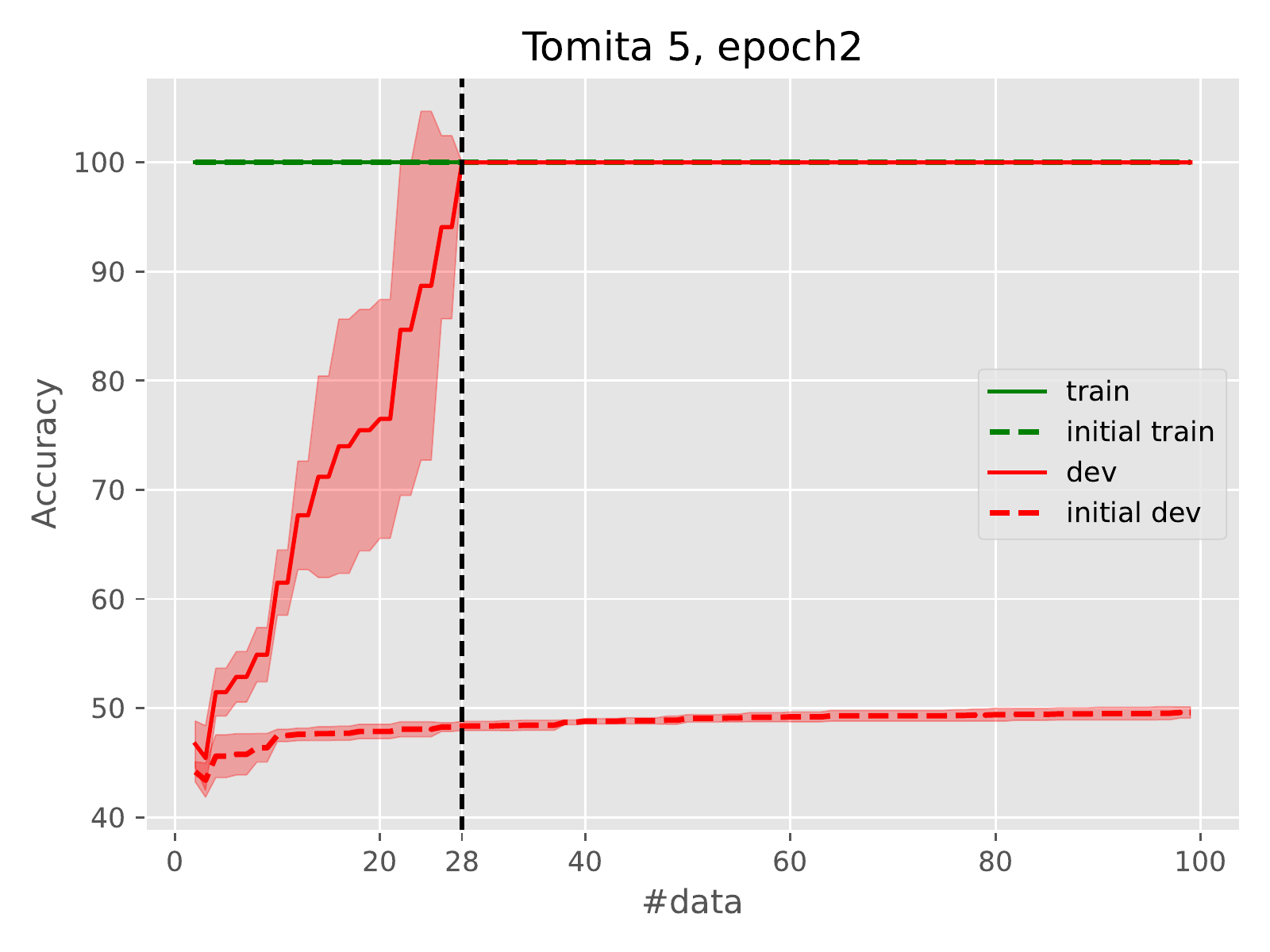}
    \includegraphics[width=0.24\linewidth]{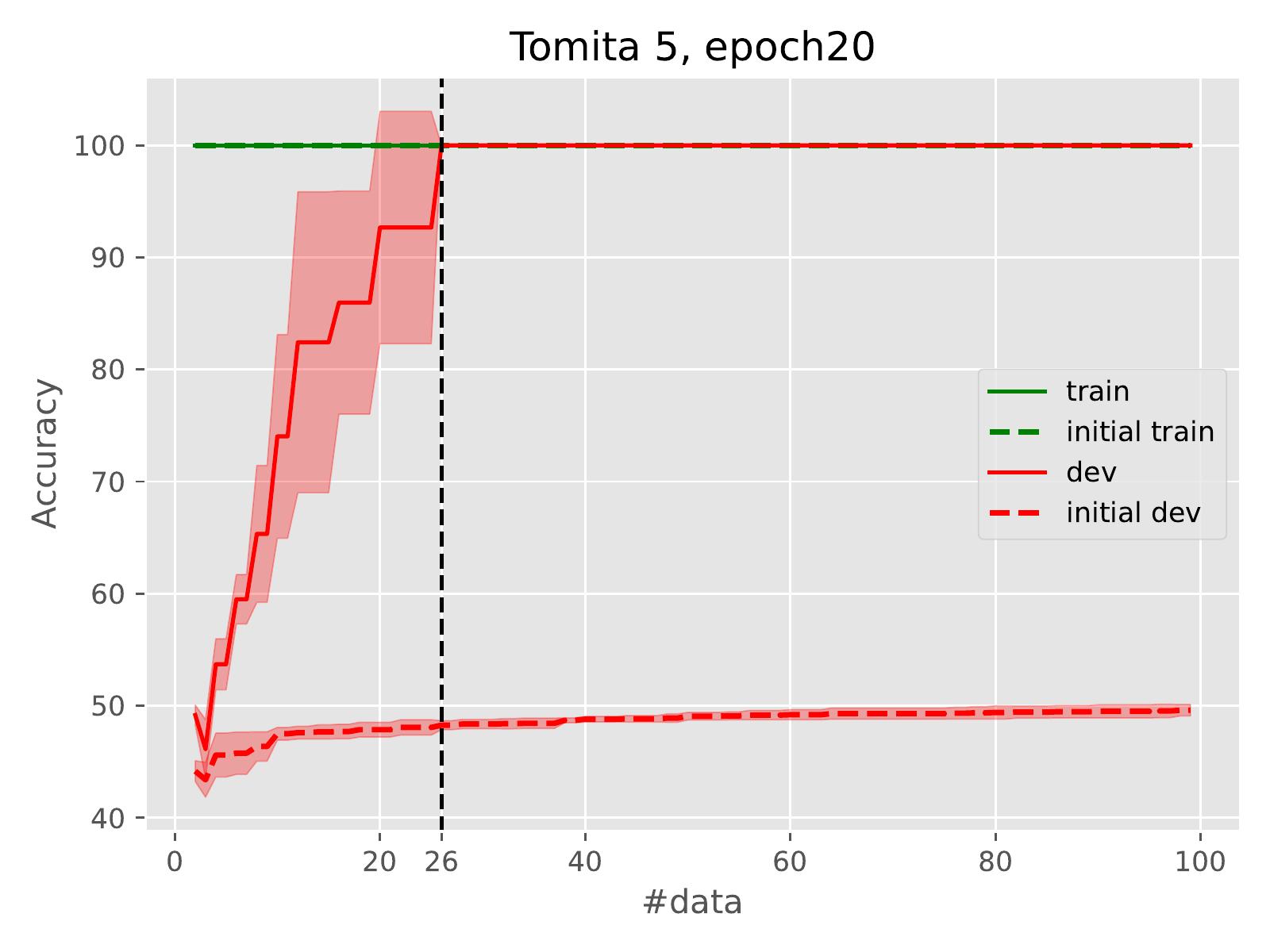}
    \includegraphics[width=0.24\linewidth]{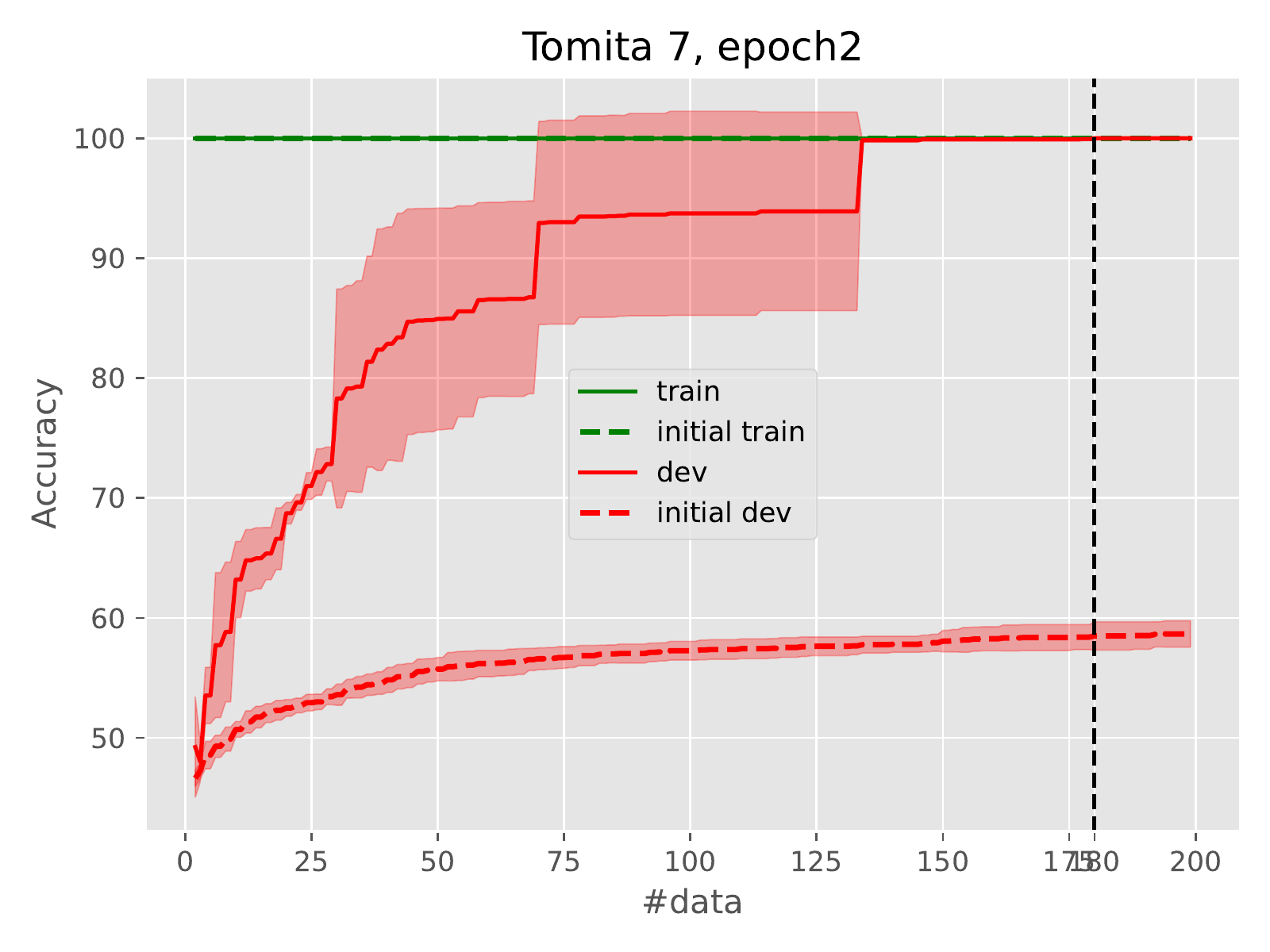}
    \includegraphics[width=0.24\linewidth]{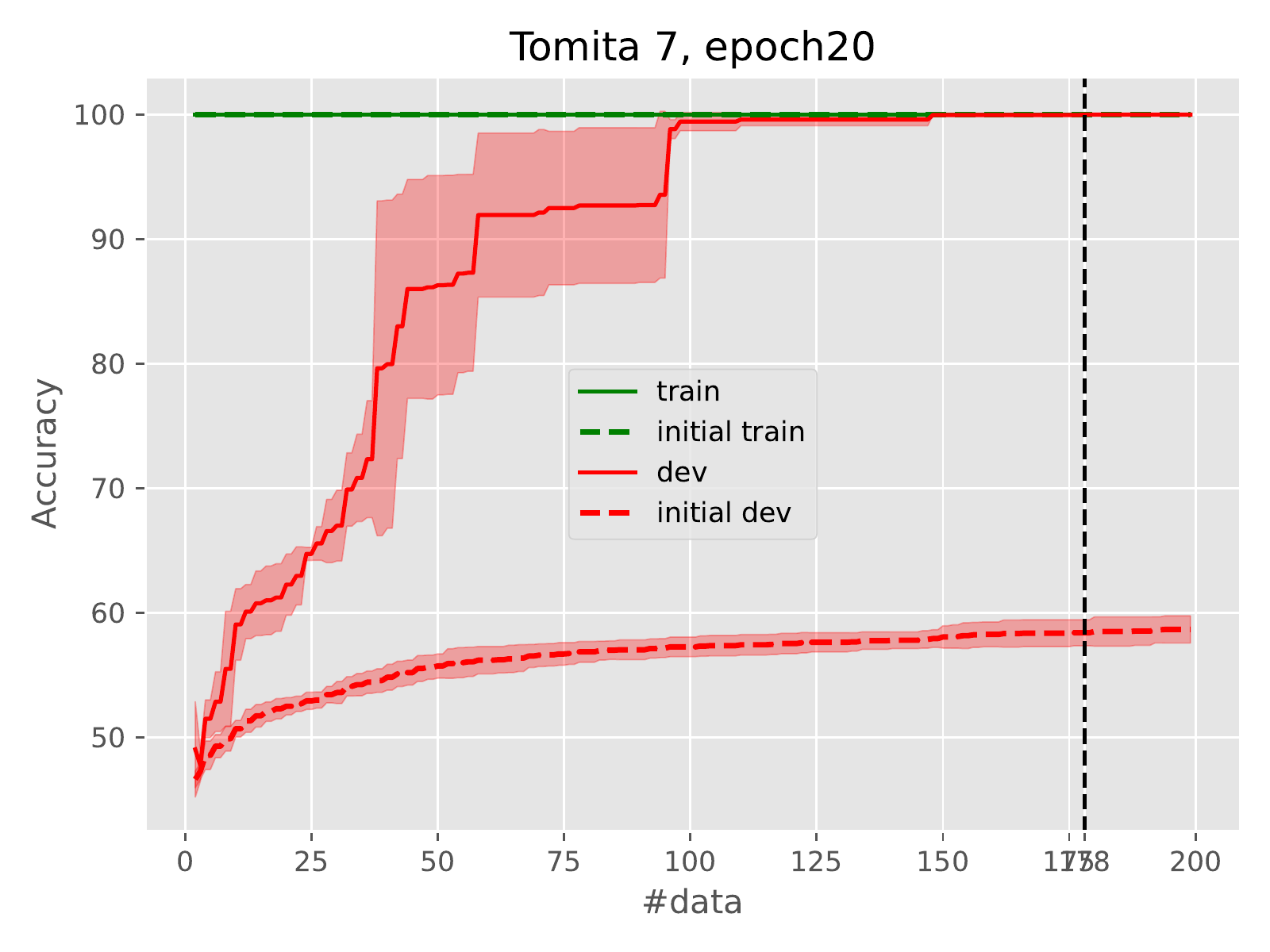}
    \caption{The implicit merging of RNN training as captured by our DFA extraction algorithm.}
    \label{fig:beyond-conv}
\end{figure*}

\section{Proof}

\begin{lemma} \label{lem:l2-cos}
If $\mathbf h_1, \mathbf h_2 \in \mathbb{R}^d$ both have unit norm, then $\norm{\mathbf h_1 - \mathbf h_2}_2^2 = 2(1 - \cos(\mathbf h_1, \mathbf h_2))$.
\end{lemma}
\begin{proof}
\begin{align*}
    \norm{\mathbf h_1 - \mathbf h_2}_2^2
    &= \norm{\mathbf h_1}_2^2 - 2\mathbf h_1^\top \mathbf h_2 + \norm{\mathbf h_2}_2^2 \\
    &= 2  - 2\mathbf h_1^\top \mathbf h_2 \\
    &= 2 (1 - \mathbf h_1^\top \mathbf h_2) \\
    &= 2 (1 - \cos(\mathbf h_1, \mathbf h_2)) .
\end{align*}
\end{proof}

Let $\Tilde{\mathbf{h}}$ be the saturated version of vector $\mathbf h$, i.e., viewing $\mathbf h$ as a function of the inputs $x$ and parameters $\theta$ \citep[cf.][]{merrill-2019-sequential},
\begin{equation*}
    \Tilde{\mathbf{h}}(x, \theta) = \lim_{\rho \to \infty} \mathbf h(x, \rho \theta) .
\end{equation*}

\begin{proposition}
Let $\mathbf{h}_1, \mathbf{h}_2 \in \mathbb{R}^d$ be two normalized state vectors, $\Tilde{\mathbf{h}}_1, \Tilde{\mathbf{h}}_2 \in \{\pm 1\}^d$ their saturated versions and assume that the RNN is $\epsilon$-saturated with respect to these states, $\norm{\mathbf h_i - \Tilde{\mathbf{h}}_i}_2 \leq \epsilon$, $i \in \{1, 2\}$. Then, if $\cos(\mathbf h_1, \mathbf h_2) \geq 1 - \kappa$ with $\sqrt{\kappa} < \sqrt{2} \left( \frac{1}{\sqrt{d}} - \epsilon \right)$,
the two vectors represent the same state on the DFA / saturated RNN ($\Tilde{\mathbf{h}}_1 = \Tilde{\mathbf{h}}_2$).
\end{proposition}

\begin{proof}
By the triangle inequality,
\begin{align*}
    \norm{\Tilde{\mathbf{h}}_1 - \Tilde{\mathbf{h}}_2}_2
    &= \norm{\Tilde{\mathbf{h}}_1 - \mathbf h_1 + \mathbf h_1 - \mathbf h_2 + \mathbf h_2 - \Tilde{\mathbf{h}}_2}_2 \\
    &\leq \norm{\Tilde{\mathbf{h}}_1 - \mathbf h_1}_2 + \norm{\mathbf h_1 - \mathbf h_2}_2 + \norm{\mathbf h_2 - \Tilde{\mathbf{h}}_2}_2 \\
    &\leq 2\epsilon + \norm{\mathbf h_1 - \mathbf h_2}_2 .
\end{align*}
Saturated RNN state vectors take discrete values in $\{-1, 1\}^d$, and thus two state vectors must be equal if the norm of their difference is $< 2/\sqrt{d}$. By the transitivity of inequalities, $\Tilde{\mathbf{h}}_1 = \Tilde{\mathbf{h}}_2$ if
\begin{equation*}
    2\epsilon + \norm{\mathbf h_1 - \mathbf h_2}_2 < \frac{2}{\sqrt{d}} .
\end{equation*}
Applying \autoref{lem:l2-cos} and using the fact that $\cos(\mathbf h_1, \mathbf h_2) \geq 1 - \kappa$,
\begin{align*}
    2\epsilon + \sqrt{2(1 - \cos(\mathbf h_1, \mathbf h_2))} &< \frac{2}{\sqrt{d}} \\
    \therefore \sqrt{\kappa} &< \sqrt{2} \left( \frac{1}{\sqrt{d}} - \epsilon \right) .
\end{align*}
\end{proof}

\end{document}